\documentclass[11pt]{article}

\usepackage{acl}

\usepackage{times}
\usepackage{latexsym}
\usepackage[T1]{fontenc}
\usepackage[utf8]{inputenc}
\usepackage{microtype}
\usepackage{amsmath}
\usepackage{amssymb}
\usepackage{amsthm}
\usepackage{booktabs}
\usepackage{multirow}
\usepackage{array}
\usepackage{graphicx}
\usepackage{float}
\usepackage{placeins}
\usepackage{tikz}
\usetikzlibrary{arrows.meta,backgrounds,calc,fit,positioning}

\newcommand{\method}{GRASP}

\newcommand{\papertitle}{Geometry-aware Residual Alignment for Scalable Pretraining Data Attribution}
\newcommand{\dataset}{\mathcal{D}}

\newcommand{\utility}{\mathcal{U}}

\newcolumntype{L}[1]{>{\raggedright\arraybackslash}p{#1}}
\newcolumntype{C}[1]{>{\centering\arraybackslash}p{#1}}
\newtheorem{theorem}{Theorem}
\newtheorem{proposition}{Proposition}

\title{\method{}: \papertitle{}}

\author{
 \textbf{Yue Min},
 \textbf{Ruining Chen},
 \textbf{Yujun Li}
\\
\\
 Wizard Intelligence Learning Lab
\\
 \small{
   \textbf{Correspondence:} \href{mailto:minyue@wizardquant.com}{minyue@wizardquant.com}, \href{mailto:liyujun@wizardquant.com}{liyujun@wizardquant.com}
 }
}

\begin{document}
\maketitle

\begin{abstract}
Scalable data attribution methods typically assign isolated utility scores to individual training examples. This prevalent additive assumption fundamentally fails to capture critical subset dynamics, including data redundancy and complementary coverage. In this work, we reframe attribution as subset-level counterfactual utility prediction and introduce \method{}, an interaction-aware surrogate. Grounded in a theoretical smoothness lower bound, \method{} explicitly models subset interactions through a quadratic geometric penalty. To achieve pretraining-scale efficiency without relying on hidden oracle tuning, we couple low-dimensional feature sketches with a strictly finite lower-confidence bound selection protocol. Extensive subset-retraining evaluations demonstrate that \method{} decisively outperforms existing scalable baselines. It more than doubles the task-level rank correlation for counterfactual subset fidelity while reducing upfront artifact construction costs by nearly an order of magnitude. Downstream diagnostics further show that this scoring mechanism transfers to language model curation and cross-domain vision selection, establishing a robust foundation for optimizing massive pretraining corpora.
\end{abstract}

\section{Introduction}

The capabilities and failure modes of large language models are fundamentally determined by their pretraining corpora \citep{raffel2020exploring,penedo2024fineweb}. Consequently, data curation has become the central driver of model performance. This process encompasses extraction, filtering, deduplication, and the construction and optimization of pretraining data mixtures \citep{min2026gem}. Yet, these structural decisions typically rely on coarse heuristics or prohibitively expensive retraining trials. Data attribution bridges this methodological gap by connecting model behavior back to the data pipeline. This connection enables targeted auditing, contamination analysis, and principled data selection under fixed compute budgets.

\begin{figure*}[t]
  \centering
  \includegraphics[width=\textwidth]{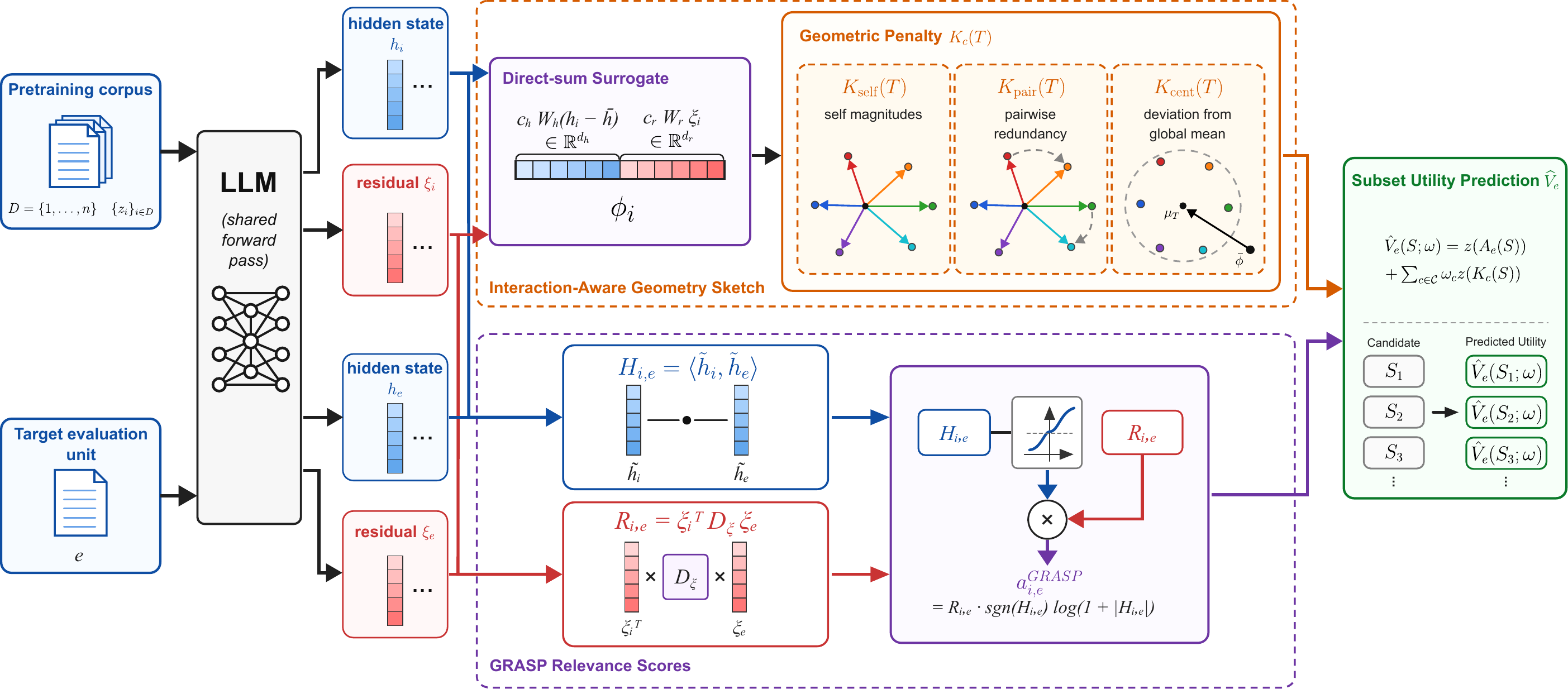}
  \caption{\textbf{Overview of the \method{} architecture.} \method{} predicts the counterfactual utility of data interventions by combining an additive relevance core (bottom) with an interaction-aware geometric penalty (top). By processing LLM hidden states and residuals into compact direct-sum sketches ($\phi_i$) and stabilized relevance scores, the method efficiently captures subset redundancy and complementary coverage. The standardized components are then linearly aggregated to evaluate entire subsets without combinatorial retraining.}
  \label{fig:grasp-pipeline}
\end{figure*}

Despite these promises, scaling attribution to guide practical curation reveals a fundamental flaw in current paradigms. Existing scalable methods, including influence functions \citep{koh2017understanding}, TracIn \citep{pruthi2020estimating}, TRAK \citep{park2023trak}, DataInf \citep{kwon2024datainf}, and LESS \citep{xia2024less}, operate by assigning a scalar relevance score to each individual training example. However, practical data curation involves retaining or removing coherent data slices rather than isolated sentences. When practitioners use pointwise scores to evaluate a data subset, they implicitly treat the utility of that subset as a simple sum of individual contributions. This additive assumption is mathematically restrictive. In real-world corpora containing near-duplicates and overlapping domains, training examples interact. Selecting highly aligned but identical examples yields diminishing returns, whereas diverse examples provide complementary coverage.

While coalition-based valuations such as Data Shapley \citep{ghorbani2019data} and datamodeling \citep{ilyas2022datamodels,wang2025data} formally capture these subset dependencies, exact coalition evaluation is computationally prohibitive for routine pretraining pipelines. This creates a critical methodological gap. Scalable attribution requires a reusable surrogate that explicitly models train-train interactions without requiring combinatorial retraining. Consequently, the evaluation of attribution methods must structurally shift from individual example ranking to \emph{subset-level counterfactual utility prediction}. A rigorous attribution method must accurately rank candidate data interventions by their empirical utility after retraining.

To address this challenge, we propose \method{}, a geometry-aware residual-alignment surrogate for scalable pretraining data attribution. Our approach is grounded in a one-step smoothness lower bound on utility improvement for weighted subset interventions. This theoretical foundation naturally yields a linear target relevance term alongside a quadratic geometric penalty. Expanding this quadratic expression formally captures the mutual dependencies omitted by additive scoring. For pretraining-scale deployment, \method{} replaces exact high-dimensional update directions with low-dimensional feature sketches. For experimental discipline, all component weights and retain/delete modes are selected on development environments before test evaluation, separating scalable subset prediction from hidden per-task tuning.

In summary, our main contributions are:
\begin{itemize}
  \setlength{\itemsep}{0.15em}
  \setlength{\parsep}{0pt}
  \setlength{\parskip}{0pt}
  \item We formalize pretraining data attribution as \emph{subset-level counterfactual utility prediction}, directly aligning the objective with practical corpus interventions to overcome the limitations of additive pointwise scoring.
  \item We derive \method{}, an interaction-aware surrogate grounded in a one-step smoothness lower bound. Its quadratic geometric penalty explicitly models data redundancy and complementary coverage without requiring combinatorial retraining.
  \item We instantiate \method{} with compact feature sketches for pretraining-scale efficiency. It reduces artifact construction time by nearly an order of magnitude, scores 100,000 subsets in 5 seconds, and more than doubles the task-level LDS rank correlation over existing scalable baselines.
\end{itemize}

\section{Problem Setup}

We study \emph{subset-level counterfactual utility prediction}: given a pretraining corpus, an evaluation target, and candidate data interventions, the objective is to rank the interventions by their realized utility after retraining.

\subsection{Pretraining Data Attribution}

Let $\dataset=\{1,\ldots,n\}$ index a pretraining corpus $\{z_i\}_{i\in\dataset}$. A data intervention retains a subset $S\subseteq\dataset$, yielding parameters $\theta_S$ after retraining. For a target $e$, we define the utility as $\utility_e(\theta)=-L_e(\theta)$, where $L_e$ is the target loss. The true subset utility is therefore
\begin{equation}
  \utility_e(S) = \utility_e(\theta_S) = -L_e(\theta_S).
  \label{eq:true-subset-utility}
\end{equation}
Pretraining attribution aims to provide a scalable surrogate $\widehat{\utility}_e:2^{\dataset}\to\mathbb{R}$ for this set function $S\mapsto\utility_e(S)$.

While traditional attribution assigns a scalar score $a_{i,e}\in\mathbb{R}$ to each example $i$ \citep{koh2017understanding,park2023trak,xia2024less}, a subset-level surrogate predicts the utility of the entire intervention $S$.

\subsection{Counterfactual Fidelity via LDS}

Following datamodeling \citep{ilyas2022datamodels,park2023trak}, we evaluate counterfactual fidelity using the Linear Datamodeling Score (LDS). Given a family of $B$ sampled interventions $\mathcal{B}=\{S_b\}_{b=1}^{B}$, LDS measures the rank correlation between predicted and realized utilities for target $e$:
\begin{equation}
  \rho_e =
  \operatorname{Spearman}_{b}
  \left(
    \widehat{\utility}_{e}(S_b),
    \utility_{e}(S_b)
  \right).
  \label{eq:lds-spearman}
\end{equation}

When relying on pointwise scores, the subset predictor degrades to an additive baseline:
\begin{equation}
  \widehat{\utility}^{\mathrm{add}}_e(S)
  = \sum_{i\in S} a_{i,e}.
  \label{eq:additive-utility}
\end{equation}
This modular form assumes independence among examples, rendering it brittle for corpora containing duplicates and domain clusters. We therefore seek a scalable surrogate incorporating a structural subset interaction term to improve $\rho_e$.

\section{Method}
\label{sec:method}

\method{} predicts the utility of a candidate data intervention without
retraining for every query. It first computes residual-alignment relevance
scores for train--target pairs, then scores a whole subset by adding
geometry-aware train-train interactions, and finally selects a fixed protocol on
development environments.

Figure~\ref{fig:grasp-pipeline} summarizes this pipeline.

\subsection{GRASP Relevance Scores}
\label{subsec:grasp-relevance}

To evaluate a data intervention, \method{} first requires a scalable measure of alignment between training and target examples. Let $\ell_i(\theta)$ be the average causal-LM training loss for sequence $i$, and let $H$ denote the Hessian or generalized Gauss--Newton (GGN) matrix of the weighted training objective. Exact influence-style alignment would use the curvature-preconditioned gradient $H^{-1}\nabla_\theta\ell_i(\theta)$, which is intractable at pretraining scale. We instead derive a proxy from last-layer causal-LM geometry.

At a supervised next-token position $t$, let $h_{i,t}$ be the final-layer hidden state, $p_{i,t}$ the predicted vocabulary distribution, and $y_{i,t+1}$ the ground-truth next-token ID. The residual
\begin{equation}
  \xi_{i,t}=p_{i,t}-\operatorname{onehot}(y_{i,t+1})
  \label{eq:token-residual}
\end{equation}
gives the exact output-layer gradient $\nabla_W\ell_{\mathrm{tok}}(i,t)=h_{i,t}\xi_{i,t}^{\top}$, up to matrix orientation. Let $P_i$ contain the last at most 32 valid supervised positions. Our default tail-mean-32 summaries are
\begin{equation}
  h_i=\frac{1}{|P_i|}\sum_{t\in P_i}h_{i,t},
  \qquad
  \xi_i=\frac{1}{|P_i|}\sum_{t\in P_i}\xi_{i,t}.
  \label{eq:sequence-pooling}
\end{equation}
For memory efficiency, each token residual retains its top-64 predicted-token coordinates together with the true-label coordinate before sparse accumulation and averaging. The sequence-level rank-1 surrogate $\nabla_W\ell_i\approx h_i\xi_i^\top$ drops within-sequence hidden--residual covariance but retains a compact alignment summary. Target summaries $h_e$ and $\xi_e$ are constructed analogously.

Under a Kronecker-factored generalized Gauss--Newton (GGN) approximation \citep{martens2015optimizing}, the inverse empirical Fisher or GGN matrix is approximated as a Kronecker product $F^{-1} \approx \Sigma_h^{-1} \otimes D_{\xi}$. Consequently, the inner product of preconditioned gradients elegantly factorizes into separate hidden and residual matches:
\begin{equation}
  \langle \nabla_W \ell_i, F^{-1} \nabla_W \ell_e \rangle 
  \approx 
  \left( h_i^\top \Sigma_h^{-1} h_e \right) \left( \xi_i^\top D_{\xi} \xi_e \right).
  \label{eq:kfac-inner-product}
\end{equation}
Motivated by this decomposition, for a target evaluation unit $e$ with representations $h_e,\xi_e$, we whiten the hidden states using global training statistics: 
\begin{equation}
  \tilde h_i=\Sigma_h^{-1/2}(h_i-\bar h),
  \qquad
  \tilde h_e=\Sigma_h^{-1/2}(h_e-\bar h),
  \label{eq:grasp-hidden-whiten}
\end{equation}
where $\Sigma_h^{-1/2}$ incorporates Tikhonov regularization. The factorized similarities are then defined as:
\begin{align}
  H_{i,e} &= \langle \tilde h_i,\tilde h_e\rangle,\\
  R_{i,e} &= \xi_i^\top D_{\xi}\xi_e .
  \label{eq:grasp-components}
\end{align}
Here, $D_{\xi} = \mathrm{diag}(\gamma+\varepsilon)^{-\alpha}$ ($\alpha \in \{1, 1/2\}$) is a stabilized diagonal inverse-covariance preconditioner: $\gamma$ contains empirical residual second moments and $\varepsilon>0$ prevents division by zero. We estimate $\gamma$ from the same sparse top-$k$ residual representation to minimize memory.

Strict adherence to local influence dictates a direct bilinear product $H_{i,e} R_{i,e}$. In practice, however, hidden-state inner products exhibit extreme heavy-tailed distributions that routinely destabilize subset evaluation. To robustify the formulation, we explicitly depart from the exact Taylor expansion and introduce a monotonic dampening function:
\begin{equation}
  a^{\mathrm{GRASP}}_{i,e}
  =
  R_{i,e}\,
  \operatorname{sgn}(H_{i,e})\log(1+|H_{i,e}|).
  \label{eq:grasp-score}
\end{equation}
This odd, strictly increasing log-transformation preserves both the ordinal ranking and the directional support (sign) of the representation match, while safely truncating the disproportionate leverage of geometric outliers. 

Finally, the first-order relevance of a weighted intervention subset $T$ is computed additively:
\begin{equation}
  A^{\mathrm{GRASP}}_e(T)
  =
  \sum_{i\in T} w_i a^{\mathrm{GRASP}}_{i,e}.
  \label{eq:grasp-subset-relevance}
\end{equation}
\method{} adopts $A^{\mathrm{GRASP}}_e(T)$ as its foundational linear predictor, which is subsequently augmented with geometry-aware interaction penalties to model subset redundancy.

\subsection{Interaction-Aware Subset Utility}
\label{subsec:interaction-aware}

Let $g_e=\nabla_\theta L_e(\theta)$ be the target-loss gradient at a reference checkpoint. Using the training-objective curvature $H$ defined above, implicit differentiation yields the influence-directed update $u_i=H^{-1}\nabla_\theta\ell_i(\theta)$ for example $i$, up to a shared scale (Appendix~\ref{app:theory}). A standard scalar attribution method scores a subset $S$ additively via this first-order alignment:
\begin{equation}
  \widehat{\utility}^{\mathrm{add}}_e(S)
  =
  \sum_{i\in S} a_{i,e},
  \qquad
  a_{i,e}\approx \langle g_e,u_i\rangle .
  \label{eq:method-additive}
\end{equation}
However, this modular form assumes strict independence, fundamentally failing to capture redundancy or complementarity within a subset. To formally model subset dynamics, we must examine the geometry of the combined update, which naturally incurs a curvature penalty.

\begin{theorem}[One-step interaction lower bound]
\label{thm:smoothness-bound}
Fix parameters $\theta$ and target $e$. Let $g_e=\nabla_\theta L_e(\theta)$, step size $\eta\ge 0$, and update direction $u_i\in\mathbb{R}^{\dim(\theta)}$. For a weighted set $T$ with weights $w_i$, define the combined update $D_T=\sum_{i\in T}w_i u_i$. If $L_e$ is $\beta_e$-smooth (with respect to the Euclidean norm) between $\theta$ and $\theta-\eta D_T$, the target utility improvement $\Delta_e(T)=\utility_e(\theta-\eta D_T)-\utility_e(\theta)$ satisfies:
\begin{equation}
  \begin{aligned}
    \Delta_e(T)
    &\ge
    \eta \sum_{i\in T}w_i\langle g_e,u_i\rangle \\
    &\quad -
    \frac{\beta_e\eta^2}{2}
    \left\|\sum_{i\in T}w_i u_i\right\|^2 .
  \end{aligned}
  \label{eq:smoothness-lower-bound}
\end{equation}
\end{theorem}

Equation~\ref{eq:smoothness-lower-bound} is a tractable one-step instance of a more general endpoint smoothness characterization. For completed retraining, the same lower-bound structure holds with the realized endpoint displacement; replacing that displacement by $\eta D_T$ introduces an explicit trajectory-mismatch term (Proposition~\ref{prop:endpoint-smoothness}). Expanding the local quadratic penalty isolates the pairwise interaction:
\begin{equation}
  \begin{aligned}
    \left\|\sum_{i\in T}w_i u_i\right\|^2
    &= \sum_{i\in T}w_i^2\|u_i\|^2 \\
    &\quad + 2\sum_{\substack{i<j\\i,j\in T}}
    w_iw_j\langle u_i,u_j\rangle .
  \end{aligned}
  \label{eq:quadratic-expansion}
\end{equation}
This reveals that mutually aligned examples ($\langle u_i,u_j\rangle > 0$) incur a strictly positive off-diagonal penalty. In terms of discrete curvature, the second difference of the set function satisfies $\Delta_i\Delta_j \widehat{\utility}^{\mathrm{int}}_e \propto - w_iw_j\langle u_i,u_j\rangle$. Thus, aligned directions exhibit diminishing marginal returns, while structurally opposing directions can be complementary (see Proposition~\ref{prop:aligned-submodularity}).

While Theorem~\ref{thm:smoothness-bound} dictates the necessity of the quadratic penalty, calculating the exact coefficient $\beta_e\eta^2/2$ is intractable for pretraining-scale interventions, as $\beta_e$ and $\eta$ fluctuate with subset norm and model scale. We therefore treat the theorem as a structural blueprint: a linear relevance term regularized by a non-negative geometric penalty. By replacing the intractable scalar with a tunable protocol parameter $\lambda_e \ge 0$, we arrive at the idealized interaction-aware surrogate:
\begin{equation}
  \widehat{\utility}^{\mathrm{int}}_e(T)
  =
  \eta A^{u}_e(T) - \lambda_e K^{u}(T),
  \label{eq:ideal-interaction-surrogate}
\end{equation}
where $A^{u}_e(T)=\sum_{i\in T}w_i\langle g_e,u_i\rangle$ captures the additive utility, and $K^{u}(T)=\|D_T\|^2$. In the implemented predictor, $\lambda_e$ is absorbed into a finite set of component weights calibrated on development environments.

\begin{table*}[t]
  \centering
  \setlength{\tabcolsep}{3.8pt}
  \begin{tabular*}{\textwidth}{@{\extracolsep{\fill}}l*{4}{cc}@{}}
    \toprule
    & \multicolumn{2}{c}{BM25}
    & \multicolumn{2}{c}{TracIn}
    & \multicolumn{2}{c}{TRAK}
    & \multicolumn{2}{c}{\method{}} \\
    \cmidrule(lr){2-3}\cmidrule(lr){4-5}\cmidrule(lr){6-7}\cmidrule(lr){8-9}
    Benchmark & $\rho$ & Pos. & $\rho$ & Pos. & $\rho$ & Pos. & $\rho$ & Pos. \\
    \midrule
    BasicSkills Common Knowledge & 0.062 & 0.520 & \underline{0.114} & \underline{0.656} & 0.059 & 0.551 & \textbf{0.312} & \textbf{0.973} \\
    BasicSkills Logical Reasoning & 0.029 & 0.500 & \underline{0.220} & \underline{0.703} & 0.167 & 0.672 & \textbf{0.457} & \textbf{0.980} \\
    OpenBookQA & 0.064 & 0.508 & 0.151 & 0.578 & \underline{0.182} & \underline{0.582} & \textbf{0.350} & \textbf{0.855} \\
    CommonsenseQA & -0.089 & 0.510 & \underline{0.163} & \underline{0.680} & -0.062 & 0.495 & \textbf{0.260} & \textbf{0.870} \\
    SciQ & 0.187 & 0.609 & 0.223 & 0.609 & \underline{0.242} & \underline{0.641} & \textbf{0.424} & \textbf{0.906} \\
    \midrule
    Average & 0.051 & 0.529 & \underline{0.174} & \underline{0.645} & 0.118 & 0.588 & \textbf{0.361} & \textbf{0.917} \\
    \bottomrule
  \end{tabular*}
  \caption{\textbf{Counterfactual subset-utility fidelity on full-dataset benchmarks.} \method{} consistently outperforms all baselines in predicting empirical retraining outcomes. We report the task-level Spearman rank correlation ($\rho$) and the fraction of evaluation units exhibiting positive LDS correlation (Pos.). Higher is better for both metrics. Best and second-best results are \textbf{bolded} and \underline{underlined}, respectively.}
  \label{tab:main-lds}
\end{table*}

\subsection{Sketched Subset Predictor}
\label{subsec:sketched-predictor}

Directly storing high-dimensional update directions $u_i \in \mathbb{R}^{\dim(\theta)}$ for every pretraining row is computationally intractable. Under the Kronecker-factored view, the true geometric interaction $\langle u_i,u_j \rangle$ takes the form of a product kernel between hidden-state and residual similarities. However, explicitly constructing the corresponding tensor-product feature $\psi_i=(W_h(h_i-\bar h))\otimes(W_r\xi_i)$ incurs a prohibitive $O(d_h d_r)$ memory footprint. To achieve linear scaling, \method{} deliberately trades exact multiplicative cross-terms for an additive direct-sum surrogate via feature concatenation:
\begin{align}
  \phi_i &= [c_h W_h(h_i-\bar h);\; c_r W_r \xi_i] \in \mathbb{R}^{d_h+d_r}.
  \label{eq:feature-sketch}
\end{align}
Here, $W_h$ is the centered whitening map from Equation~\ref{eq:grasp-hidden-whiten}, $W_r$ is the residual scaling induced by the stabilized preconditioner $D_\xi$ (or its sparse top-$k$ approximation), and $c_h,c_r$ are fixed scaling constants. While this direct-sum kernel intentionally deviates from an unbiased product kernel, it ensures fast $O(|T|d)$ subset sweeps and linear storage space. Appendix~\ref{app:theory} formalizes the local approximation bound and resulting sketch error of this relaxation.

Using this compact sketch, we operationalize the blueprint of Theorem~\ref{thm:smoothness-bound}. The scalar relevance $A_e(T) = A^{\mathrm{GRASP}}_e(T)$ is retrieved separately. For any candidate intervention set $T$, the raw geometric penalty is $K_{\mathrm{raw}}(T) = \|\sum_{i\in T}w_i\phi_i\|^2$. To grant the predictor finer expressivity over self-regularization and cross-sample redundancy, we orthogonally decompose and center this geometry:
\begin{align}
  K_{\mathrm{self}}(T) &=
  \sum_{i\in T}w_i^2\|\phi_i\|^2,\\
  K_{\mathrm{pair}}(T) &=
  K_{\mathrm{raw}}(T)-K_{\mathrm{self}}(T),\\
  K_{\mathrm{cent}}(T) &=
  \left\|\sum_{i\in T}w_i(\phi_i-\bar{\phi})\right\|^2 ,
  \label{eq:ia-components}
\end{align}
where $\bar{\phi}=n^{-1}\sum_{i\in\dataset}\phi_i$ is the global feature mean. 

Finally, constructing a unified surrogate requires reconciling the relative scales of the linear relevance and the decomposed quadratic penalties. As dictated by Theorem~\ref{thm:smoothness-bound}, the true interaction scale depends on a target-specific and subset-dependent smoothness constant $\beta_e$, which cannot be analytically computed. For each component, let $\mu$ and $\sigma$ be its empirical mean and standard deviation on a fixed calibration subset family. These moments put relevance and geometry on comparable scales, and test interventions never determine their own normalization. We then use
\begin{equation}
  \widehat{V}_e(T;\omega)
  =
  z(A_e(T))
  +
  \sum_{c\in\mathcal{C}}\omega_c\,z(K_c(T)).
  \label{eq:component-predictor}
\end{equation}
Here, $\mathcal{C}$ contains the selected geometric components (omitting zero-variance ones), and $z(x)=(x-\mu)/\sigma$. By mapping all targets into a standardized scale space, \method{} allows a single, dataset-agnostic weight vector $\omega$ (calibrated strictly on development environments; see Appendix~\ref{app:implementation}) to implicitly govern the theoretical $\lambda_e$ interaction scale. Once scores and sketches are pre-computed, evaluating a subset costs $O(|T|d)$ time with only $O(nd)$ global storage.

\section{Experiments}
\label{sec:experiments}

The core premise of subset-level attribution is counterfactual fidelity. In this section, we first evaluate this primary capability using rigorous subset-retraining protocols (Section~\ref{sec:exp-subset-fidelity}). We then conduct ablation studies to isolate the mechanistic contributions of our interaction-aware components (Section~\ref{sec:exp-mechanism}). Finally, we demonstrate the operational utility of \method{} across downstream applications, including language model curation, cross-domain vision transfer, and semantic retrieval (Section~\ref{sec:exp-applications}). Setting-specific configurations for all experiments are detailed in Appendix~\ref{app:experimental-details}.

\subsection{Counterfactual Subset-Utility Evaluation}
\label{sec:exp-protocols}
\label{sec:exp-subset-fidelity}

Table~\ref{tab:main-lds} reports the Linear Datamodeling Score (LDS) evaluation, which directly tests counterfactual fidelity by measuring the rank correlation between predicted subset utilities and their empirical outcomes after retraining. To ensure a rigorous comparison, we strictly control the intervention candidate pool and the subset-retraining recipe across all methods. Our experimental testbed utilizes a 50M-parameter causal decoder-only Transformer \citep{vaswani2017attention,radford2019language} trained on a 10B-token pool. This pool is randomly sampled from a comprehensive Common Crawl corpus curated via standard scalable data pipelines \citep{raffel2020exploring,penedo2024refinedweb,penedo2024fineweb,li2024dclm}. The 50M-parameter setting is an evaluation-scale choice rather than an algorithmic limit: counterfactual LDS requires rebuilding every baseline and retraining every candidate subset under a matched protocol. Complete metric definitions and baseline details are provided in Appendix~\ref{app:experimental-details}.

The empirical results demonstrate that \method{} strictly dominates all baseline methods across both metrics on every evaluated benchmark. Notably, \method{} more than doubles the average task-level Spearman correlation compared to the strongest alternative, improving $\rho$ from 0.174 (TracIn) to 0.361. Furthermore, it substantially elevates the positive-LDS fraction from 0.645 to 0.917, indicating highly reliable directional predictions across diverse intervention samples. This substantial performance gap provides the central evidence that explicitly capturing train-train interactions yields a strictly superior surrogate for ranking concrete pretraining interventions.

We additionally measure retraining variability on 30 fixed SciQ subsets using three optimizer seeds while holding subset masks and attribution predictions fixed. \method{} retains a task-level $\rho$ of $0.414\pm0.046$, compared with $0.235\pm0.041$ for the strongest baseline in this audit (TRAK-5); Appendix Table~\ref{tab:multiseed-variance} gives all rows. This control indicates that the observed ranking gap is robust to the measured training-run variation.

\begin{figure*}[t]
  \centering
  \includegraphics[width=0.92\textwidth]{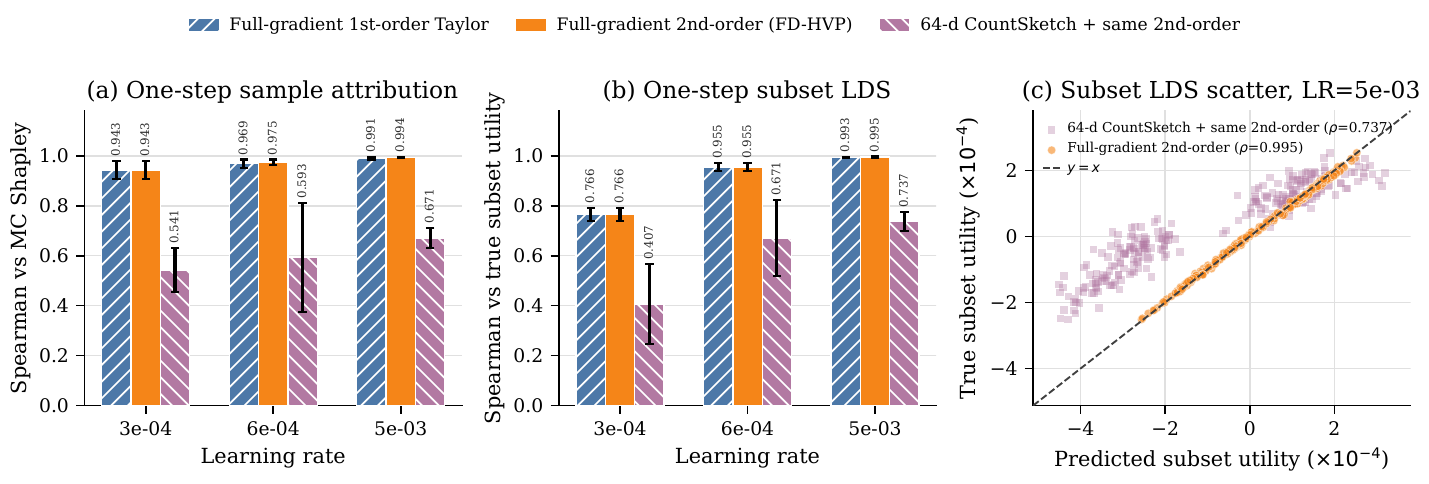}
  \caption{\textbf{One-step counterfactual fidelity on CC text.} Panels (a--b) report mean Spearman correlation over two initialization seeds and disjoint 16-example candidate batches; black whiskers are standard errors. The first-order full-gradient Taylor predictor (blue) and its finite-difference-HVP second-order extension (orange) are compared with a 64-dimensional CountSketch projected-gradient predictor given the same second-order correction (purple). Panel (a) uses Monte Carlo Shapley values from 1,000 permutations of a fixed 16-example game; panel (b) uses 128 fixed-size eight-example subsets. Panel (c) plots the two second-order predictions against realized utility at learning rate $5\times10^{-3}$; the dashed line is $y=x$.}
  \label{fig:one-step-grasp-fidelity}
\end{figure*}
\begin{table}[t]
  \centering
  \begin{tabular*}{\columnwidth}{@{\extracolsep{\fill}}lccc@{}}
    \toprule
    Method & Build (d) & State & Sweep (s) \\
    \midrule
    TracIn & 2.2 & 1 & 0.086 \\
    TRAK & 2.2 & 640 & 159.8 \\
    InRun-DS & 10.0 & 1 & 0.086 \\
    Ret.-DS & 12.2 & 0 & $1.05{\times}10^6$ \\
    \method{} & 0.27 & 384 & 5.08 \\
    \bottomrule
  \end{tabular*}
  \caption{\textbf{\method{} achieves efficient artifact construction and subset evaluation.} Build reports days, State reports stored representation floats per example for the evaluated configuration, and Sweep reports seconds for 100,000 candidate subsets.}
  \label{tab:subset-scoring-speed}
\end{table}
Practical pretraining attribution demands both cheap artifact construction and rapid subset evaluation. Table~\ref{tab:subset-scoring-speed} uncouples these costs, reporting upfront construction in days, stored representation floats per example, and cached 100k-subset sweeps in seconds. Representation counts include checkpoints or component sketches used by the evaluated configuration; deployed cache sizes, a storage-matched TRAK control, and replay throughput are reported in Appendix~\ref{app:baseline-controls}. While scalar methods like TracIn enable fast sweeps once artifacts exist, faithful per-evaluation InRun Data Shapley requires an expensive target-conditioned replay. In contrast, \method{} requires 0.27 days for artifact construction and executes a 100k-subset sweep in 5.08 seconds, approximately 31$\times$ faster than the TRAK cached aggregation benchmark.

\begin{figure}[t]
  \centering
  \includegraphics[width=\columnwidth]{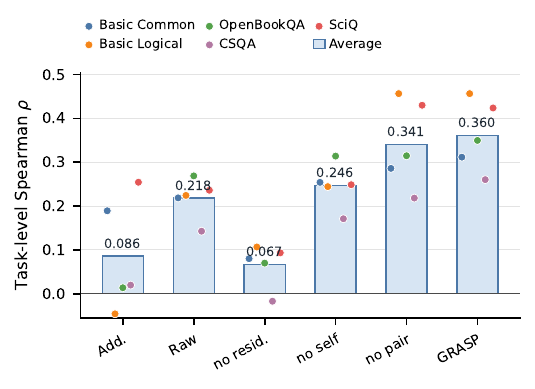}
  \caption{\textbf{Impact of interaction components on subset prediction fidelity.} We ablate \method{} components across all primary LDS benchmarks. Bars show mean task-level Spearman correlation ($\rho$), and colored points show individual tasks.}
  \label{fig:component-ablation}
\end{figure}

\subsection{Mechanistic Analysis and Ablations}
\label{sec:exp-mechanism}

To understand the source of these performance gains, we isolate the mechanistic contributions of individual \method{} components. Detailed one-step ground-truth construction and ablation protocols are given in Appendix~\ref{app:experimental-details}, with an additional cached protocol-selection diagnostic in Appendix Table~\ref{tab:cached-grasp-protocol-sweep}.

Figure~\ref{fig:one-step-grasp-fidelity} shows that the \method{} relevance core tracks both Monte Carlo Shapley values and one-step subset utilities. Furthermore, the interaction surrogate substantially improves the high-learning-rate regime where subset updates are less linear. A projected-gradient control with an identical interaction correction remains significantly weaker, indicating that the performance gain does not stem from adding a generic pairwise correction to an arbitrary score.

The long-horizon LDS results exhibit the same pattern on complete subset retraining interventions. Relative to an additive-only predictor, incorporating the fixed interaction components raises the macro-average task-level Spearman from 0.086 to 0.360 and pairwise intervention accuracy from 0.528 to 0.623. Figure~\ref{fig:component-ablation} further shows the dataset-level structure of this gain: residual alignment is the most critical component family, and removing it degrades every benchmark. Meanwhile, self terms provide a consistent secondary contribution, whereas hidden geometry exhibits more task-dependent behavior. Appendix Table~\ref{tab:nonadditive-controls} separately compares the additive self-norm, pair-only, and pair-shuffled controls to isolate strictly cross-example gains.

\subsection{Downstream Applications and Transferability}
\label{sec:exp-applications}

Beyond baseline counterfactual fidelity, we evaluate how the interaction-aware scoring principle performs in practical data pipelines and entirely distinct domains.

\paragraph{Fixed-compute language model curation.} 

To evaluate practical utility in pretraining pipelines, we deploy \method{} to select a high-quality training pool prior to model initialization. Table~\ref{tab:dclm-scaleup} reports completed held-out Common Crawl and SciQ evaluations under the same fixed-compute training schedule for a 1B-parameter model and matched 25B-token budget. \method{} achieves the lowest held-out CC loss and the best SciQ accuracy among the displayed full-budget diagnostic rows. This is application-scale evidence that the scoring artifacts can support a larger model, not a direct 1B LDS fidelity test. We report early training-log snapshots and compact cross-architecture cleaning diagnostics separately in Appendix Tables~\ref{tab:dclm-logged-loss-snapshot} and~\ref{tab:cc-cleaning-best-observed}.

\begin{table}[t]
  \centering
  \setlength{\tabcolsep}{2pt}
  \begin{tabular*}{\columnwidth}{@{\extracolsep{\fill}}lccc@{}}
    \toprule
    Method & CC loss $\downarrow$ & PPL $\downarrow$ & SciQ acc. $\uparrow$ \\
    \midrule
    Random & 2.9644 & 19.38 & \underline{0.3203} \\
    InRun-1 & 2.9632 & 19.36 & 0.3086 \\
    InRun-2 & \underline{2.9550} & \underline{19.20} & 0.3047 \\
    \method{} & \textbf{2.9421} & \textbf{18.96} & \textbf{0.3398} \\
    \bottomrule
  \end{tabular*}
  \caption{\textbf{Held-out evaluation for fixed-compute CC-10B curation.} All rows use the same model, optimizer, selected-pool size, and training budget.}
  \label{tab:dclm-scaleup}
\end{table}

\paragraph{Cross-domain transfer: Vision data selection.} 
\begin{table}[t]
  \centering
  \setlength{\tabcolsep}{0pt}
  \begin{tabular*}{\columnwidth}{@{\extracolsep{\fill}}lccccc@{}}
    \toprule
    Method & AUC & 20\% & 40\% & 60\% & 80\% \\
    \midrule
    Random & 53.92 & 50.9 & 63.2 & 70.0 & 74.5 \\
    KNN & 73.77 & 34.1 & 48.7 & 64.7 & 74.3 \\
    Infl. & 75.59 & 45.8 & 64.4 & \underline{72.2} & \textbf{76.9} \\
    TRAK & 55.38 & 29.9 & 43.4 & 52.8 & 65.0 \\
    DataM. & 54.34 & 47.6 & 64.1 & 71.4 & 74.1 \\
    InRun-1 & \underline{80.28} & \underline{55.6} & \underline{68.9} & 71.9 & 75.5 \\
    InRun-2 & 80.04 & 52.8 & 66.4 & \textbf{73.4} & 75.9 \\
    \method{} & \textbf{80.32} & \textbf{56.1} & \textbf{69.6} & 71.8 & \underline{76.1} \\
    \bottomrule
  \end{tabular*}
  \caption{\textbf{Cross-domain transferability on vision data selection.} AUC evaluates noisy-label detection; remaining columns report selected-data accuracy at each retention budget. KNN: KNN-Shapley, Infl.: Influence, DataM.: Datamodels. All entries are percentages; higher is better.}
  \label{tab:vision-sanity}
\end{table}
To verify whether this scoring principle generalizes beyond language modeling, we evaluate our formulation in the standard Data Shapley vision setting \citep{ghorbani2019data}: ImageNet-pretrained ResNet18 \citep{deng2009imagenet,he2016deep}, CIFAR-10 \citep{krizhevsky2009learning} 1k subset, and 10\% label noise. Table~\ref{tab:vision-sanity} shows that \method{} achieves the best AUROC and strong selected-data accuracy across budgets.

\paragraph{Interpretability and Safety Extensions.} 
We explore the semantic expressiveness of our attribution signal in Appendix~\ref{app:weborganizer-qualitative}, using WebOrganizer labels \citep{wettig2025organize} to show that \method{} provides interpretable, category-level valuations that strongly complement traditional lexical retrieval. Additionally, Appendix Table~\ref{tab:safety-detection} extends our framework to safety-oriented applications, demonstrating that the same foundational scoring mechanism effectively identifies and filters misleading information during fine-tuning.
\begin{table}[t]
  \centering
  \setlength{\tabcolsep}{5pt}
  \begin{tabular*}{\columnwidth}{@{\extracolsep{\fill}}cccc@{}}
    \toprule
    Retained & BM25 $\rho$ & Add. $\rho$ & \method{} $\rho$ \\
    \midrule
    0.20 & -0.153 & -0.292 & \textbf{0.516} \\
    0.50 & 0.297 & 0.285 & \textbf{0.503} \\
    0.80 & 0.230 & 0.165 & \textbf{0.490} \\
    0.90 & 0.187 & 0.254 & \textbf{0.447} \\
    0.95 & -0.688 & 0.065 & \textbf{0.279} \\
    0.99 & -0.075 & \textbf{0.210} & \textbf{0.210} \\
    \bottomrule
  \end{tabular*}
  \caption{\textbf{Regime-dependence of interaction effects.} Performance is evaluated on the SciQ retention-ratio environment. BM25 serves as a lexical baseline, while Add.\ represents the relevance-only additive core of \method{}. All method columns report task-level Spearman $\rho$. The 0.95 and 0.99 rows represent high-retention stress regimes, evaluating behavior under extreme sparsity of removed data.}
  \label{tab:discussion-ratio}
\end{table}

\section{Discussion}

\paragraph{When do subset interactions matter most?}
The necessity of interaction-aware attribution is inherently tied to the intervention regime. In practical curation scenarios where a selected candidate subset diverges significantly from the reference corpus, ranking quality is no longer dictated by isolated relevance alone. Instead, redundant examples induce gradient saturation along shared update directions, whereas complementary examples jointly maximize semantic coverage. Table~\ref{tab:discussion-ratio} quantifies this phenomenon within the SciQ LDS environment. Across substantial and moderate retention fractions, our geometric surrogate consistently outperforms its additive counterpart by modeling these subset interactions. Conversely, at an extreme 0.99 retention rate, the candidate subsets differ by only a negligible slice; here, the perturbation approaches a first-order additive limit. Crucially, \method{} elegantly reduces to its additive core in this regime, inherently suppressing unstable interaction noise rather than amplifying it. This behavior aligns perfectly with the theoretical formulation of the quadratic term: it dominates when geometric interactions dictate the relative ordering of interventions, and gracefully diminishes when interventions are near-identical micro-perturbations.

\paragraph{Semantic interpretability of the geometric surrogate.}
While \method{} is primarily formulated as a robust counterfactual predictor, its underlying geometry simultaneously yields highly interpretable corpus-level insights. As detailed in the WebOrganizer analysis (Appendix~\ref{app:weborganizer-qualitative}), our attribution scores and traditional lexical retrieval capture fundamentally complementary structures. For topic-level retrieval, \method{} functions competitively as a standalone metric, but fusing it with BM25 yields the optimal normalized Recall@1000. An analogous complementarity emerges for format labels when bridging \method{} with query overlap (Table~\ref{tab:weborganizer-domain}). This synergy confirms that the learned attribution signal captures deeper semantic utility rather than merely mirroring surface-level lexical overlap. Furthermore, the checkpoint-wise category visualization (Figure~\ref{fig:weborganizer-contribution}) exposes explicit, signed source preferences: \method{} attributes positive expected value to knowledge-dense, article-like regions for the SciQ target, while actively penalizing boilerplate-heavy formats such as notices, support pages, and FAQs. Ultimately, these interpretable dynamics demonstrate that the geometric surrogate extends far beyond subset ranking, offering a principled mechanism for transparent corpus auditing and dataset valuation.

\section{Related Work}
\label{app:related-work}

Our work is situated at the intersection of data attribution, scalable influence estimation, subset counterfactual modeling, and pretraining data auditing. We shift the core objective from assigning isolated scores to individual examples toward accurately predicting the empirical effects of concrete data interventions.

\paragraph{Data attribution and data valuation.}
Training data attribution investigates how predictions depend on specific examples. Methods like influence functions \citep{koh2017understanding}, representer points \citep{yeh2018representer}, and RelatIF \citep{barshan2020relatif} estimate local training contributions. Alternatively, Shapley-value methods formalize data value through marginal contributions across training coalitions \citep{ghorbani2019data}. While interaction-aware, Shapley aggressively averages these coalition dynamics into a static scalar per example, and exact estimation is computationally prohibitive. In-Run Data Shapley mitigates this cost along a single training trajectory \citep{wang2025data}. In contrast, we explicitly retain reusable train-train geometry to predict the utility of concrete subset interventions.

\paragraph{Scalable gradient-based attribution.}
Recent advancements make gradient attribution viable for large-scale models. TracIn integrates gradient inner products across training checkpoints \citep{pruthi2020estimating}. Other methods scale Hessian-based influence \citep{schioppa2022scaling}, analyze LLM generalization \citep{grosse2023studying}, derive efficient estimators for tuned models \citep{kwon2024datainf}, accelerate attribution via projected features \citep{park2023trak}, and adapt gradient similarity for instruction tuning \citep{xia2024less}. Closest to our domain, \citet{chang2024scalable} scale gradient-based influence to LLM pretraining. While these methods successfully deliver scalable first-order signals, we focus exclusively on subset-level counterfactual fidelity rather than retrieving isolated influential examples.


\paragraph{Pretraining data auditing and selection.}
Extensive corpus analyses reveal substantial duplication, toxic content, and benchmark contamination in pretraining datasets \citep{elazar2024whats}. Existing data selection techniques utilize gradient-derived relevance \citep{xia2024less} or pretraining-scale tracing \citep{chang2024scalable} to surface beneficial examples. \method{} addresses a distinct structural challenge: predicting the relative downstream utility of candidate pretraining subsets. In this setting, redundancy, coverage, and subset interactions are first-class concerns, rendering purely lexical retrieval or independent example scoring mathematically insufficient.

\section{Conclusion}

This work reframes pretraining data attribution from isolated example scoring to subset-level counterfactual utility prediction. \method{} turns this view into a scalable surrogate by combining a smoothness-motivated geometric penalty, low-dimensional sketches, and development-only protocol selection. Across subset-retraining benchmarks, the method improves counterfactual ranking fidelity over scalable baselines, while the application diagnostics show that the same interaction-aware signal remains useful beyond the core LDS setting. These results suggest that subset geometry is a practical ingredient for auditing and optimizing massive pretraining corpora.

\section*{Limitations}

Our theoretical formulation relies on a one-step smoothness lower bound. This is an approximation rather than a strict mathematical guarantee for full non-convex Transformer pretraining. Consequently, we rely on subset-retraining LDS as an empirical counterfactual test rather than claiming an exact analytical match.

Empirically, our validation centers on a 50M-parameter model and a 10B-token Common Crawl pool. While this scale enables strictly controlled interventions, it does not automatically guarantee identical behavior for frontier-scale models or radically different corpus mixtures. Scaling these findings remains an important empirical direction.

Furthermore, \method{} utilizes low-dimensional feature sketches to ensure computational scalability. This compression inevitably discards fine-grained attention interactions and long-range training dependencies, rendering the method a scalable surrogate rather than a lossless influence calculation.

Finally, our lower-confidence bound protocols are explicitly target-conditioned. We do not claim that \method{} provides a single universal data-quality score for all tasks. Additionally, extreme data interventions with high-overlap deletion ratios may necessitate specialized ratio-aware scaling adjustments.
\section*{Ethical Considerations}

Data attribution can support responsible corpus auditing by identifying data
regions that are influential for target behaviors, harmful outputs, or
benchmark-like evaluations. At the same time, attribution scores are
approximations and should not be treated as definitive evidence that a single
document caused a model behavior. We recommend reporting uncertainty, negative
controls, and failure modes whenever attribution is used for data governance.

Pretraining corpora may contain copyrighted, private, toxic, or benchmark-
contaminating content. Releasing raw influential examples or high-resolution
attribution tables can expose sensitive text or facilitate reconstruction of
corpus membership. Any release of data examples, scores, or source-retrieval
outputs should follow the underlying dataset licenses and privacy constraints.
\bibliography{references}

@inproceedings{koh2017understanding,
  title     = {Understanding Black-box Predictions via Influence Functions},
  author    = {Pang Wei Koh and Percy Liang},
  booktitle = {Proceedings of the 34th International Conference on Machine Learning},
  pages     = {1885--1894},
  year      = {2017},
  editor    = {Precup, Doina and Teh, Yee Whye},
  volume    = {70},
  series    = {Proceedings of Machine Learning Research},
  month     = {06--11 Aug},
  publisher = {PMLR},
  url       = {https://proceedings.mlr.press/v70/koh17a.html}
}

@inproceedings{vaswani2017attention,
 author = {Vaswani, Ashish and Shazeer, Noam and Parmar, Niki and Uszkoreit, Jakob and Jones, Llion and Gomez, Aidan N and Kaiser, \L ukasz and Polosukhin, Illia},
 booktitle = {Advances in Neural Information Processing Systems},
 editor = {I. Guyon and U. Von Luxburg and S. Bengio and H. Wallach and R. Fergus and S. Vishwanathan and R. Garnett},
 pages = {},
 publisher = {Curran Associates, Inc.},
 title = {Attention is All you Need},
 url = {https://proceedings.neurips.cc/paper_files/paper/2017/file/3f5ee243547dee91fbd053c1c4a845aa-Paper.pdf},
 volume = {30},
 year = {2017}
}

@misc{radford2019language,
  title = {Language Models Are Unsupervised Multitask Learners},
  author={Radford, Alec and Wu, Jeffrey and Child, Rewon and Luan, David and Amodei, Dario and Sutskever, Ilya},
  year = {2019},
  howpublished = {OpenAI technical report},
  url = {https://cdn.openai.com/better-language-models/language_models_are_unsupervised_multitask_learners.pdf}
}

@inproceedings{yeh2018representer,
 author = {Yeh, Chih-Kuan and Kim, Joon and Yen, Ian En-Hsu and Ravikumar, Pradeep K},
 booktitle = {Advances in Neural Information Processing Systems},
 editor = {S. Bengio and H. Wallach and H. Larochelle and K. Grauman and N. Cesa-Bianchi and R. Garnett},
 pages = {},
 publisher = {Curran Associates, Inc.},
 title = {Representer Point Selection for Explaining Deep Neural Networks},
 url = {https://proceedings.neurips.cc/paper_files/paper/2018/file/8a7129b8f3edd95b7d969dfc2c8e9d9d-Paper.pdf},
 volume = {31},
 year = {2018}
}

@inproceedings{barshan2020relatif,
  title     = {RelatIF: Identifying Explanatory Training Samples via Relative Influence},
  author    = {Barshan, Elnaz and Brunet, Marc-Etienne and Dziugaite, Gintare Karolina},
  booktitle = {Proceedings of the Twenty Third International Conference on Artificial Intelligence and Statistics},
  pages     = {1899--1909},
  year      = {2020},
  editor    = {Chiappa, Silvia and Calandra, Roberto},
  volume    = {108},
  series    = {Proceedings of Machine Learning Research},
  month     = {26--28 Aug},
  publisher = {PMLR},
  url       = {https://proceedings.mlr.press/v108/barshan20a.html}
}

@inproceedings{ghorbani2019data,
  title     = {Data Shapley: Equitable Valuation of Data for Machine Learning},
  author    = {Ghorbani, Amirata and Zou, James},
  booktitle = {Proceedings of the 36th International Conference on Machine Learning},
  pages     = {2242--2251},
  year      = {2019},
  editor    = {Chaudhuri, Kamalika and Salakhutdinov, Ruslan},
  volume    = {97},
  series    = {Proceedings of Machine Learning Research},
  month     = {09--15 Jun},
  publisher = {PMLR},
  url       = {https://proceedings.mlr.press/v97/ghorbani19c.html}
}

@inproceedings{mihaylov2018openbookqa,
  title     = {Can a Suit of Armor Conduct Electricity? A New Dataset for Open Book Question Answering},
  author    = {Mihaylov, Todor  and
               Clark, Peter  and
               Khot, Tushar  and
               Sabharwal, Ashish},
  editor    = {Riloff, Ellen  and
               Chiang, David  and
               Hockenmaier, Julia  and
               Tsujii, Jun{'}ichi},
  booktitle = {Proceedings of the 2018 Conference on Empirical Methods in Natural Language Processing},
  month     = oct # {-} # nov,
  year      = {2018},
  address   = {Brussels, Belgium},
  publisher = {Association for Computational Linguistics},
  url       = {https://aclanthology.org/D18-1260/},
  doi       = {10.18653/v1/D18-1260},
  pages     = {2381--2391}
}

@inproceedings{talmor2019commonsenseqa,
  title     = {{C}ommonsense{QA}: A Question Answering Challenge Targeting Commonsense Knowledge},
  author    = {Talmor, Alon  and
               Herzig, Jonathan  and
               Lourie, Nicholas  and
               Berant, Jonathan},
  editor    = {Burstein, Jill  and
               Doran, Christy  and
               Solorio, Thamar},
  booktitle = {Proceedings of the 2019 Conference of the North {A}merican Chapter of the Association for Computational Linguistics: Human Language Technologies, Volume 1 (Long and Short Papers)},
  month     = jun,
  year      = {2019},
  address   = {Minneapolis, Minnesota},
  publisher = {Association for Computational Linguistics},
  url       = {https://aclanthology.org/N19-1421/},
  doi       = {10.18653/v1/N19-1421},
  pages     = {4149--4158}
}

@article{raffel2020exploring,
  author  = {Colin Raffel and Noam Shazeer and Adam Roberts and Katherine Lee and Sharan Narang and Michael Matena and Yanqi Zhou and Wei Li and Peter J. Liu},
  title   = {Exploring the Limits of Transfer Learning with a Unified Text-to-Text Transformer},
  journal = {Journal of Machine Learning Research},
  year    = {2020},
  volume  = {21},
  number  = {140},
  pages   = {1--67},
  url     = {http://jmlr.org/papers/v21/20-074.html}
}

@inproceedings{wang2025data,
 author = {Wang, Jiachen (Tianhao) and Mittal, Prateek and Song, Dawn and Jia, Ruoxi},
 booktitle = {International Conference on Learning Representations},
 editor = {Y. Yue and A. Garg and N. Peng and F. Sha and R. Yu},
 pages = {12358--12395},
 title = {Data Shapley in One Training Run},
 url = {https://proceedings.iclr.cc/paper_files/paper/2025/file/20fdaf67581e6d7157376d1ed584040a-Paper-Conference.pdf},
 volume = {2025},
 year = {2025}
}

@inproceedings{pruthi2020estimating,
 author = {Pruthi, Garima and Liu, Frederick and Kale, Satyen and Sundararajan, Mukund},
 booktitle = {Advances in Neural Information Processing Systems},
 editor = {H. Larochelle and M. Ranzato and R. Hadsell and M.F. Balcan and H. Lin},
 pages = {19920--19930},
 publisher = {Curran Associates, Inc.},
 title = {Estimating Training Data Influence by Tracing Gradient Descent},
 url = {https://proceedings.neurips.cc/paper_files/paper/2020/file/e6385d39ec9394f2f3a354d9d2b88eec-Paper.pdf},
 volume = {33},
 year = {2020}
}

@inproceedings{schioppa2022scaling,
  title={Scaling up influence functions},
  author={Schioppa, Andrea and Zablotskaia, Polina and Vilar, David and Sokolov, Artem},
  booktitle={Proceedings of the AAAI Conference on Artificial Intelligence},
  volume={36},
  pages={8179--8186},
  year={2022}
}

@misc{grosse2023studying,
  title={Studying Large Language Model Generalization with Influence Functions}, 
  author={Roger Grosse and Juhan Bae and Cem Anil and Nelson Elhage and Alex Tamkin and Amirhossein Tajdini and Benoit Steiner and Dustin Li and Esin Durmus and Ethan Perez and Evan Hubinger and Kamilė Lukošiūtė and Karina Nguyen and Nicholas Joseph and Sam McCandlish and Jared Kaplan and Samuel R. Bowman},
  year={2023},
  eprint={2308.03296},
  archivePrefix={arXiv},
  primaryClass={cs.LG},
  url={https://arxiv.org/abs/2308.03296}, 
}

@misc{park2023trak,
  title={TRAK: Attributing Model Behavior at Scale}, 
  author={Sung Min Park and Kristian Georgiev and Andrew Ilyas and Guillaume Leclerc and Aleksander Madry},
  year={2023},
  eprint={2303.14186},
  archivePrefix={arXiv},
  primaryClass={stat.ML},
  url={https://arxiv.org/abs/2303.14186}, 
}

@inproceedings{kwon2024datainf,
 author = {Kwon, Yongchan and Wu, Eric and Wu, Kevin and Zou, James Y},
 booktitle = {International Conference on Learning Representations},
 editor = {B. Kim and Y. Yue and S. Chaudhuri and K. Fragkiadaki and M. Khan and Y. Sun},
 pages = {21921--21942},
 title = {DataInf: Efficiently Estimating Data Influence in LoRA-tuned LLMs and Diffusion Models},
 url = {https://proceedings.iclr.cc/paper_files/paper/2024/file/5e84a0f233611a1dc8fb794dc52415a3-Paper-Conference.pdf},
 volume = {2024},
 year = {2024}
}

@inproceedings{xia2024less,
  author = {Xia, Mengzhou and Malladi, Sadhika and Gururangan, Suchin and Arora, Sanjeev and Chen, Danqi},
  title = {LESS: selecting influential data for targeted instruction tuning},
  year = {2024},
  publisher = {JMLR.org},
  booktitle = {Proceedings of the 41st International Conference on Machine Learning},
  articleno = {2221},
  numpages = {29},
  location = {Vienna, Austria},
  series = {ICML'24}
}

@inproceedings{penedo2024fineweb,
 author = {Penedo, Guilherme and Kydl\'{\i}\v{c}ek, Hynek and allal, Loubna Ben and Lozhkov, Anton and Mitchell, Margaret and Raffel, Colin and Von Werra, Leandro and Wolf, Thomas},
 booktitle = {Advances in Neural Information Processing Systems},
 doi = {10.52202/079017-0970},
 editor = {A. Globerson and L. Mackey and D. Belgrave and A. Fan and U. Paquet and J. Tomczak and C. Zhang},
 pages = {30811--30849},
 publisher = {Curran Associates, Inc.},
 title = {The FineWeb Datasets: Decanting the Web for the Finest Text Data at Scale},
 url = {https://proceedings.neurips.cc/paper_files/paper/2024/file/370df50ccfdf8bde18f8f9c2d9151bda-Paper-Datasets_and_Benchmarks_Track.pdf},
 volume = {37},
 year = {2024}
}

@inproceedings{penedo2024refinedweb,
 author = {Penedo, Guilherme and Malartic, Quentin and Hesslow, Daniel and Cojocaru, Ruxandra and Alobeidli, Hamza and Cappelli, Alessandro and Pannier, Baptiste and Almazrouei, Ebtesam and Launay, Julien},
 booktitle = {Advances in Neural Information Processing Systems},
 editor = {A. Oh and T. Naumann and A. Globerson and K. Saenko and M. Hardt and S. Levine},
 pages = {79155--79172},
 publisher = {Curran Associates, Inc.},
 title = {The RefinedWeb Dataset for Falcon LLM: Outperforming Curated Corpora with Web Data Only},
 url = {https://proceedings.neurips.cc/paper_files/paper/2023/file/fa3ed726cc5073b9c31e3e49a807789c-Paper-Datasets_and_Benchmarks.pdf},
 volume = {36},
 year = {2023}
}

@inproceedings{li2024dclm,
 author = {Li, Jeffrey and Fang, Alex and Smyrnis, Georgios and Ivgi, Maor and Jordan, Matt and Gadre, Samir and Bansal, Hritik and Guha, Etash and Keh, Sedrick and Arora, Kushal and Garg, Saurabh and Xin, Rui and Muennighoff, Niklas and Heckel, Reinhard and Mercat, Jean and Chen, Mayee and Gururangan, Suchin and Wortsman, Mitchell and Albalak, Alon and Bitton, Yonatan and Nezhurina, Marianna and Abbas, Amro and Hsieh, Cheng-Yu and Ghosh, Dhruba and Gardner, Josh and Kilian, Maciej and Zhang, Hanlin and Shao, Rulin and Pratt, Sarah and Sanyal, Sunny and Ilharco, Gabriel and Daras, Giannis and Marathe, Kalyani and Gokaslan, Aaron and Zhang, Jieyu and Chandu, Khyathi and Nguyen, Thao and Vasiljevic, Igor and Kakade, Sham and Song, Shuran and Sanghavi, Sujay and Faghri, Fartash and Oh, Sewoong and Zettlemoyer, Luke and Lo, Kyle and El-Nouby, Alaaeldin and Pouransari, Hadi and Toshev, Alexander and Wang, Stephanie and Groeneveld, Dirk and Soldaini, Luca and Koh, Pang Wei and Jitsev, Jenia and Kollar, Thomas and Dimakis, Alexandros G. and Carmon, Yair and Dave, Achal and Schmidt, Ludwig and Shankar, Vaishaal},
 booktitle = {Advances in Neural Information Processing Systems},
 doi = {10.52202/079017-0455},
 editor = {A. Globerson and L. Mackey and D. Belgrave and A. Fan and U. Paquet and J. Tomczak and C. Zhang},
 pages = {14200--14282},
 publisher = {Curran Associates, Inc.},
 title = {DataComp-LM: In search of the next generation of training sets for language models},
 url = {https://proceedings.neurips.cc/paper_files/paper/2024/file/19e4ea30dded58259665db375885e412-Paper-Datasets_and_Benchmarks_Track.pdf},
 volume = {37},
 year = {2024}
}

@inproceedings{ilyas2022datamodels,
  title     = {Datamodels: Understanding Predictions with Data and Data with Predictions},
  author    = {Ilyas, Andrew and Park, Sung Min and Engstrom, Logan and Leclerc, Guillaume and Madry, Aleksander},
  booktitle = {Proceedings of the 39th International Conference on Machine Learning},
  pages     = {9525--9587},
  year      = {2022},
  editor    = {Chaudhuri, Kamalika and Jegelka, Stefanie and Song, Le and Szepesvari, Csaba and Niu, Gang and Sabato, Sivan},
  volume    = {162},
  series    = {Proceedings of Machine Learning Research},
  month     = {17--23 Jul},
  publisher = {PMLR},
  url       = {https://proceedings.mlr.press/v162/ilyas22a.html}
}

@inproceedings{gu2025olmes,
  title     = {{OLMES}: A Standard for Language Model Evaluations},
  author    = {Gu, Yuling  and
               Tafjord, Oyvind  and
               Kuehl, Bailey  and
               Haddad, Dany  and
               Dodge, Jesse  and
               Hajishirzi, Hannaneh},
  editor    = {Chiruzzo, Luis  and
               Ritter, Alan  and
               Wang, Lu},
  booktitle = {Findings of the Association for Computational Linguistics: NAACL 2025},
  month     = apr,
  year      = {2025},
  address   = {Albuquerque, New Mexico},
  publisher = {Association for Computational Linguistics},
  url       = {https://aclanthology.org/2025.findings-naacl.282/},
  doi       = {10.18653/v1/2025.findings-naacl.282},
  pages     = {5020--5048},
  isbn      = {979-8-89176-195-7}
}

@misc{olmo2025olmo3,
  title={Olmo 3}, 
  author={Team Olmo},
  year={2026},
  eprint={2512.13961},
  archivePrefix={arXiv},
  primaryClass={cs.CL},
  url={https://arxiv.org/abs/2512.13961}, 
}

@inproceedings{elazar2024whats,
 author = {Elazar, Yanai and Bhagia, Akshita and Magnusson, Ian and Ravichander, Abhilasha and Schwenk, Dustin and Suhr, Alane and Walsh, Pete and Groeneveld, Dirk and Soldaini, Luca and Singh, Sameer and Hajishirzi, Hannaneh  and Smith, Noah and Dodge, Jesse},
 booktitle = {International Conference on Learning Representations},
 editor = {B. Kim and Y. Yue and S. Chaudhuri and K. Fragkiadaki and M. Khan and Y. Sun},
 pages = {7735--7790},
 title = {What\textquotesingle s In My Big Data?},
 url = {https://proceedings.iclr.cc/paper_files/paper/2024/file/1f7336fd66b6e6e63d1801fdd5930a5a-Paper-Conference.pdf},
 volume = {2024},
 year = {2024}
}

@inproceedings{chang2024scalable,
 author = {Chang, Tyler and Rajagopal, Dheeraj and Bolukbasi, Tolga and Dixon, Lucas and Tenney, Ian},
 booktitle = {International Conference on Learning Representations},
 editor = {Y. Yue and A. Garg and N. Peng and F. Sha and R. Yu},
 pages = {40976--40997},
 title = {Scalable Influence and Fact Tracing for Large Language Model Pretraining},
 url = {https://proceedings.iclr.cc/paper_files/paper/2025/file/65798a76cc176c29b6bfefe84b0a03ff-Paper-Conference.pdf},
 volume = {2025},
 year = {2025}
}

@inproceedings{welbl2017crowdsourcing,
  title     = {Crowdsourcing Multiple Choice Science Questions},
  author    = {Welbl, Johannes  and
               Liu, Nelson F.  and
               Gardner, Matt},
  editor    = {Derczynski, Leon  and
               Xu, Wei  and
               Ritter, Alan  and
               Baldwin, Tim},
  booktitle = {Proceedings of the 3rd Workshop on Noisy User-generated Text},
  month     = sep,
  year      = {2017},
  address   = {Copenhagen, Denmark},
  publisher = {Association for Computational Linguistics},
  url       = {https://aclanthology.org/W17-4413/},
  doi       = {10.18653/v1/W17-4413},
  pages     = {94--106}
}

@article{nemhauser1978analysis,
  title = {An Analysis of Approximations for Maximizing Submodular Set Functions---I},
  author = {Nemhauser, George L. and Wolsey, Laurence A. and Fisher, Marshall L.},
  journal = {Mathematical Programming},
  volume = {14},
  number = {1},
  pages = {265--294},
  year = {1978},
  doi = {10.1007/BF01588971},
  url = {https://doi.org/10.1007/BF01588971}
}

@article{gretton2012kernel,
  author  = {Arthur Gretton and Karsten M. Borgwardt and Malte J. Rasch and Bernhard Sch{{\"o}}lkopf and Alexander Smola},
  title   = {A Kernel Two-Sample Test},
  journal = {Journal of Machine Learning Research},
  year    = {2012},
  volume  = {13},
  number  = {25},
  pages   = {723--773},
  url     = {http://jmlr.org/papers/v13/gretton12a.html}
}

@article{dasgupta2003elementary,
  author = {Dasgupta, Sanjoy and Gupta, Anupam},
  title = {An elementary proof of a theorem of Johnson and Lindenstrauss},
  journal = {Random Structures \& Algorithms},
  volume = {22},
  number = {1},
  pages = {60-65},
  doi = {https://doi.org/10.1002/rsa.10073},
  url = {https://onlinelibrary.wiley.com/doi/abs/10.1002/rsa.10073},
  eprint = {https://onlinelibrary.wiley.com/doi/pdf/10.1002/rsa.10073},
  year = {2003}
}

@article{achlioptas2003database,
  title = {Database-friendly random projections: Johnson-Lindenstrauss with binary coins},
  author = {Dimitris Achlioptas},
  journal = {Journal of Computer and System Sciences},
  volume = {66},
  number = {4},
  pages = {671-687},
  year = {2003},
  note = {Special Issue on PODS 2001},
  issn = {0022-0000},
  doi = {https://doi.org/10.1016/S0022-0000(03)00025-4},
  url = {https://www.sciencedirect.com/science/article/pii/S0022000003000254}
}

@inproceedings{martens2015optimizing,
  title = {Optimizing Neural Networks with {Kronecker}-factored Approximate Curvature},
  author = {Martens, James and Grosse, Roger},
  booktitle = {Proceedings of the 32nd International Conference on Machine Learning},
  series = {Proceedings of Machine Learning Research},
  volume = {37},
  pages = {2408--2417},
  publisher = {PMLR},
  year = {2015},
  url = {https://proceedings.mlr.press/v37/martens15.html}
}

@inproceedings{he2016deep,
  title = {Deep Residual Learning for Image Recognition},
  author = {He, Kaiming and Zhang, Xiangyu and Ren, Shaoqing and Sun, Jian},
  booktitle = {Proceedings of the IEEE Conference on Computer Vision and Pattern Recognition},
  pages = {770--778},
  year = {2016},
  doi = {10.1109/CVPR.2016.90},
  url = {https://openaccess.thecvf.com/content_cvpr_2016/html/He_Deep_Residual_Learning_CVPR_2016_paper.html}
}

@inproceedings{deng2009imagenet,
  title = {{ImageNet}: A Large-Scale Hierarchical Image Database},
  author = {Deng, Jia and Dong, Wei and Socher, Richard and Li, Li-Jia and Li, Kai and Fei-Fei, Li},
  booktitle = {2009 IEEE Conference on Computer Vision and Pattern Recognition},
  pages = {248--255},
  year = {2009},
  doi = {10.1109/CVPR.2009.5206848},
  url = {https://ieeexplore.ieee.org/document/5206848}
}

@techreport{krizhevsky2009learning,
  title = {Learning Multiple Layers of Features from Tiny Images},
  author = {Krizhevsky, Alex},
  institution = {University of Toronto},
  year = {2009},
  url = {https://www.cs.toronto.edu/~kriz/learning-features-2009-TR.pdf}
}

@misc{wettig2025organize,
  title = {Organize the Web: Constructing Domains Enhances Pre-Training Data Curation},
  author = {Wettig, Alexander and Lo, Kyle and Min, Sewon and Hajishirzi, Hannaneh and Chen, Danqi and Soldaini, Luca},
  year = {2025},
  eprint = {2502.10341},
  archivePrefix = {arXiv},
  primaryClass = {cs.CL},
  doi = {10.48550/arXiv.2502.10341},
  url = {https://arxiv.org/abs/2502.10341}
}

@misc{min2026gem,
  title = {{GEM}: Geometric Entropy Mixing for Optimal {LLM} Data Curation},
  author = {Min, Yue and Qiao, Ziyun and Chen, Ruining and Li, Yujun},
  year = {2026},
  eprint = {2605.26121},
  archivePrefix = {arXiv},
  primaryClass = {cs.LG},
  doi = {10.48550/arXiv.2605.26121},
  url = {https://arxiv.org/abs/2605.26121}
}

\appendix

\section{Preliminaries and Terminology}
\label{app:preliminaries}

This section collects the conceptual background and notation used throughout the paper. It distinguishes training examples from evaluation targets and separates realized counterfactual utilities from the scalable predictors used to approximate them.

\paragraph{Learning and intervention setting.}
Let $\dataset=\{1,\ldots,n\}$ index a pretraining corpus $\{z_i\}_{i\in\dataset}$. A fixed training recipe specifies the architecture, initialization distribution, optimizer, update budget, and data-order policy. For a retained subset $S\subseteq\dataset$, $\theta_S$ denotes the parameters obtained by training on $S$ under this recipe. Thus, $S$ is a data intervention, whereas $\theta_S$ is its resulting model. A retention intervention is represented directly by $S$; a deletion intervention may equivalently be described by the omitted set $T=\dataset\setminus S$. Holding the recipe and compute budget fixed prevents candidate subsets from being confounded with different optimization budgets.

The outcome associated with $S$ is counterfactual: it describes what the model would do if it were trained on $S$ rather than on another candidate subset or the full corpus. Realizing this outcome for many subsets requires a separate training run for each $S$, which makes exhaustive evaluation prohibitive.

\paragraph{Evaluation targets and utility.}
An evaluation target $e$ specifies the behavior whose dependence on training data is being studied. It may be a held-out example, a collection of examples, or a task-level evaluation set. Indices $i$ and $e$ therefore have distinct roles: $i$ indexes candidate training data, whereas $e$ indexes an evaluation unit. With $L_e(\theta)$ denoting average target loss, we use negative loss as utility:
\begin{equation}
  \utility_e(\theta)=-L_e(\theta),
  \qquad
  \utility_e(S)=\utility_e(\theta_S).
  \label{eq:prelim-utility}
\end{equation}
Higher utility therefore means lower target loss. The counterfactual object of interest is the set function $S\mapsto\utility_e(S)$.

\paragraph{Pointwise and subset-level attribution.}
Conventional pointwise attribution assigns each training example a scalar $a_{i,e}$. A common subset extension sums these scores,
\begin{equation}
  \widehat{\utility}^{\mathrm{add}}_e(S)
  =\sum_{i\in S}a_{i,e},
  \label{eq:prelim-additive}
\end{equation}
which assumes that the contribution of one example is independent of the other examples in $S$. This can miss redundancy among duplicates or overlapping documents and complementary coverage among distinct examples. Subset-level attribution instead approximates the full set function and permits the predicted contribution of one example to depend on the rest of the intervention. \method{} introduces an explicit interaction term for this purpose.

\paragraph{Scalable attribution.}
Here, \emph{scalable} means that computation is amortized over many subset queries. A method may make a one-time pass over a reference model and corpus, but it must reuse the resulting per-example statistics to score new candidate subsets rather than retraining for every query. Scalability is not equivalence to full retraining: artifact construction, storage, and repeated subset evaluation must be considered together with fidelity to realized retraining outcomes.

\paragraph{Counterfactual fidelity and LDS.}
Let $\mathcal{B}=\{S_b\}_{b=1}^{B}$ be a finite family of candidate interventions. Subset retraining gives realized utilities $\utility_e(S_b)$, and an attribution method gives predictions $\widehat{\utility}_e(S_b)$. Following datamodeling and TRAK \citep{ilyas2022datamodels,park2023trak}, LDS is their Spearman rank correlation:
\begin{equation}
  \rho_e=\operatorname{Spearman}_{b}
  \left(\widehat{\utility}_e(S_b),\utility_e(S_b)\right).
  \label{eq:prelim-lds}
\end{equation}
A high LDS means that the surrogate preserves the empirical ordering of candidate interventions even if its absolute utility scale is not calibrated. ``Linear Datamodeling Score'' is the established name of the protocol; using it as a rank diagnostic does not require the evaluated surrogate to be linear.

\paragraph{Causal-LM quantities.}
For sequence $i$ and supervised next-token position $t$, $p_{i,t}$ is the model softmax vector, $y_{i,t+1}$ the ground-truth token ID, and $h_{i,t}$ the final-layer hidden state. The token loss and output residual are
\begin{equation}
  \begin{aligned}
    \ell_{\mathrm{tok}}(i,t;\theta)
    &=-\log p_{i,t}[y_{i,t+1}],\\
    \xi_{i,t}
    &=p_{i,t}-\operatorname{onehot}(y_{i,t+1}).
  \end{aligned}
\end{equation}
The output-projection gradient has outer-product form $h_{i,t}\xi_{i,t}^{\top}$ up to matrix orientation. For valid supervised positions $P_i$, the sequence loss is $\ell_i(\theta)=|P_i|^{-1}\sum_{t\in P_i}\ell_{\mathrm{tok}}(i,t;\theta)$. Section~\ref{subsec:grasp-relevance} specifies how \method{} pools these token-level quantities before whitening and sketching them.

\section{Method Details}
\label{app:theory}

This appendix provides the formal derivations supporting the interaction-aware surrogate utilized for the counterfactual evaluation in Section~\ref{sec:experiments}. Our primary objective is to delineate which components of our formulation are direct mathematical consequences of objective smoothness and which originate from the approximation conditions imposed by the low-dimensional sketch.

\paragraph{Local influence direction.}
The local update direction $u_i$ introduced in the main text is formally instantiated through classical influence function analysis. Consider the weighted regularized empirical risk
\begin{equation}
  R(\theta,\alpha)
  =
  \frac{1}{n}\sum_{i=1}^{n}\alpha_i \ell_i(\theta)
  + \lambda_{\mathrm{reg}} r(\theta),
  \label{eq:weighted-risk}
\end{equation}
where $\alpha_i$ represents the weight assigned to training example $i$. Assume $R$ is twice continuously differentiable. Let $\theta^\star(\alpha)$ be a local minimizer such that $\nabla_\theta R(\theta^\star(\alpha), \alpha) = 0$. If the Hessian $H = \nabla_\theta^2 R(\theta^\star(\alpha), \alpha)$ is strictly positive definite, the Implicit Function Theorem guarantees the existence of a differentiable solution branch. Differentiating the stationarity condition with respect to $\alpha_i$ yields $\partial\theta^\star(\alpha)/\partial\alpha_i = -n^{-1}H^{-1}\nabla_\theta \ell_i(\theta^\star(\alpha))$. 

Recall that the evaluation utility is defined as $\utility_e(\theta) = -L_e(\theta)$. Applying the chain rule evaluated at the reference checkpoint $\theta$ gives the marginal utility change:
\begin{equation}
  \frac{\partial \utility_e}{\partial \alpha_i}
  =
  \frac{1}{n}
  \left\langle \nabla_\theta L_e(\theta), H^{-1} \nabla_\theta \ell_i(\theta) \right\rangle.
  \label{eq:first-order-influence}
\end{equation}
This derivation motivates defining the local influence-style update direction as $u_i = H^{-1}\nabla_\theta \ell_i(\theta)$, absorbing the shared $1/n$ scaling factor into the global learning rate.

\paragraph{Smoothness lower bound.}
Building upon this local linear approximation, we derive a certified lower bound for subset utility improvements. Assume the target loss function $L_e$ is $\beta_e$-smooth. By definition, for any parameter update $\Delta$, the following inequality holds:
\begin{equation}
  L_e(\theta+\Delta)
  \le
  L_e(\theta)
  + \langle \nabla_\theta L_e(\theta), \Delta \rangle
  + \frac{\beta_e}{2}\|\Delta\|^2.
\end{equation}
We model the subset intervention as an aggregated gradient step defined by $\Delta = -\eta D_T$, where $\eta \ge 0$ is the step size. Substituting this into the smoothness condition and utilizing the utility definition $U_e = -L_e$ yields:
\begin{equation}
  \begin{aligned}
    U_e(\theta-\eta D_T) - U_e(\theta)
    &\ge
    \eta\langle \nabla_\theta L_e(\theta), D_T \rangle \\
    &\quad
    - \frac{\beta_e \eta^2}{2}\|D_T\|^2.
  \end{aligned}
\end{equation}
Finally, substituting the subset aggregate direction $D_T = \sum_{i\in T}w_i u_i$ directly recovers the one-step interaction lower bound presented in Theorem~\ref{thm:smoothness-bound}. Crucially, this bound relies solely on the $\beta_e$-smoothness of the target objective and does not require global convexity.

\begin{proposition}[Endpoint smoothness with trajectory mismatch]
\label{prop:endpoint-smoothness}
Let $\theta_{\mathrm{ref}}$ be a shared reference parameter and $\theta_T$ the endpoint obtained after completing the fixed training recipe on subset $T$. Define
\[
  v_T=\theta_{\mathrm{ref}}-\theta_T,
  \qquad
  g_e=\nabla_\theta L_e(\theta_{\mathrm{ref}}).
\]
If $L_e$ has a $\beta_{e,T}$-Lipschitz gradient on the segment joining $\theta_{\mathrm{ref}}$ and $\theta_T$, then the realized endpoint utility satisfies
\begin{equation}
  \Delta_e^{\mathrm{full}}(T)
  \ge
  \langle g_e,v_T\rangle
  -\frac{\beta_{e,T}}{2}\|v_T\|^2,
  \label{eq:endpoint-smoothness}
\end{equation}
where $\Delta_e^{\mathrm{full}}(T)=\utility_e(\theta_T)-\utility_e(\theta_{\mathrm{ref}})$. Moreover, for a reusable proxy $\widehat v_T=\eta D_T$ satisfying $\|v_T-\widehat v_T\|\le\varepsilon_T$,
\begin{align}
  \Delta_e^{\mathrm{full}}(T)
  \ge{}&
  \langle g_e,\widehat v_T\rangle
  -\frac{\beta_{e,T}}{2}\|\widehat v_T\|^2 \notag\\
  &-\varepsilon_T
  \|g_e-\beta_{e,T}\widehat v_T\|
  -\frac{\beta_{e,T}}{2}\varepsilon_T^2.
  \label{eq:trajectory-mismatch}
\end{align}
\end{proposition}

\begin{proof}
Applying the smoothness inequality along the segment from $\theta_{\mathrm{ref}}$ to $\theta_T=\theta_{\mathrm{ref}}-v_T$ and negating the target loss gives Equation~\ref{eq:endpoint-smoothness}. Write $v_T=\widehat v_T+d_T$ with $\|d_T\|\le\varepsilon_T$. Expanding its right-hand side yields the proxy quadratic plus $\langle g_e-\beta_{e,T}\widehat v_T,d_T\rangle-\beta_{e,T}\|d_T\|^2/2$. Cauchy--Schwarz and the bound on $d_T$ give Equation~\ref{eq:trajectory-mismatch}.
\end{proof}

This endpoint result is optimizer- and convexity-agnostic because the completed dynamics are represented by $v_T$. It does not assume that the trajectory mismatch $\varepsilon_T$ vanishes. Controlling $\varepsilon_T$ for arbitrary non-convex pretraining would require additional assumptions on the training dynamics; our LDS experiments instead measure end-to-end ranking fidelity against the realized $\theta_T$. This separates trajectory approximation from the sketch approximation analyzed below.

\paragraph{Set-function curvature and submodularity.}
The idealized surrogate defined in Equation~\ref{eq:ideal-interaction-surrogate} offers a precise set-theoretic interpretation of subset ``interaction.'' Let $F_e(T)=\widehat{\utility}^{\mathrm{int}}_e(T)$. For any elements $i,j\notin T$ with $i\ne j$, the discrete second difference characterizes the marginal interaction:
\begin{equation}
  \begin{aligned}
    \Delta_i\Delta_j F_e(T)
    &=
    F_e(T\cup\{i,j\}) - F_e(T\cup\{i\})\\
    &\quad - F_e(T\cup\{j\}) + F_e(T).
  \end{aligned}
  \label{eq:discrete-second-difference}
\end{equation}
Direct algebraic expansion yields:
\begin{equation}
  \Delta_i\Delta_j F_e(T)
  =
  -2\lambda_e w_iw_j\langle u_i,u_j\rangle.
  \label{eq:ia-discrete-curvature}
\end{equation}
This establishes a critical geometric intuition: aligned update directions naturally produce negative discrete curvature (redundancy penalty), whereas opposing directions can yield positive curvature (complementary coverage). Consequently, \method{} functions as a signed pairwise-interaction set function. In the specific case where $w_i\ge 0$ and all pairs exhibit non-negative inner products ($\langle u_i,u_j\rangle\ge 0$), the surrogate becomes globally submodular, reflecting classical diminishing returns \citep{nemhauser1978analysis}. This conditional submodularity is formalized below.

\begin{proposition}[Local submodularity over aligned clusters]
\label{prop:aligned-submodularity}
Consider the unweighted discrete surrogate defined by
\[
  F(T)=\eta\sum_{i\in T}a_i-\lambda\left\|\sum_{i\in T}u_i\right\|^2,
  \qquad \lambda>0.
\]
Let $\mathcal{D}_{\mathrm{align}}$ be a candidate pool such that the pairwise alignment $\langle u_i,u_j\rangle\ge 0$ holds for all distinct $i,j\in\mathcal{D}_{\mathrm{align}}$. Then, the set function $F$ is submodular on $\mathcal{D}_{\mathrm{align}}$. Moreover, for any $S\subseteq T\subseteq\mathcal{D}_{\mathrm{align}}$ and $k\in\mathcal{D}_{\mathrm{align}}\setminus T$, the marginal diminishing return is strictly positive whenever $\sum_{i\in T\setminus S}\langle u_k,u_i\rangle>0$.
\end{proposition}

\begin{proof}
Let $\Delta_kF(S)=F(S\cup\{k\})-F(S)$ denote the marginal gain of adding element $k$ to set $S$. The linear relevance term contributes a constant $\eta a_k$ to every marginal gain. Expanding the quadratic penalty term yields:
\[
  \Delta_k F(S)
  =
  \eta a_k
  -\lambda\|u_k\|^2
  -2\lambda\sum_{i\in S}\langle u_k,u_i\rangle.
\]
To verify submodularity, we examine the difference in marginal gains for $S\subseteq T$:
\[
  \Delta_kF(S)-\Delta_kF(T)
  =
  2\lambda\sum_{i\in T\setminus S}\langle u_k,u_i\rangle.
\]
Under the alignment assumption $\langle u_k,u_i\rangle\ge 0$, this difference is strictly non-negative since $\lambda>0$. This guarantees $\Delta_kF(S) \ge \Delta_kF(T)$, satisfying the formal definition of diminishing returns and proving that $F$ is submodular on $\mathcal{D}_{\mathrm{align}}$. The strict inequality condition follows directly from the same expression.
\end{proof}

\paragraph{Sketched surrogate and error bounds.}
To operationalize this objective at scale, we introduce the sketched surrogate for a fixed nonnegative scale $\lambda$:
\begin{equation}
  \widehat{B}^{\mathrm{sk}}_e(T;\lambda)
  =
  \eta A_e(T)-\lambda K_{\mathrm{raw}}(T).
  \label{eq:sketched-surrogate}
\end{equation}
Let $B^u_e(T;\lambda)=\eta A^u_e(T)-\lambda K^u(T)$ denote the exact local surrogate. If the sketch formulation satisfies bounded deviations on a subset family $\mathcal{T}$ such that
\begin{align}
  |A_e(T)-A^u_e(T)| &\le \epsilon_A(T),\\
  |K_{\mathrm{raw}}(T)-K^u(T)| &\le \epsilon_K(T),
\end{align}
then applying the triangle inequality guarantees that for every $\lambda\ge 0$ and $T\in\mathcal{T}$, the subset prediction error is bounded by:
\begin{equation}
  \left|
  B^u_e(T;\lambda)-\widehat{B}^{\mathrm{sk}}_e(T;\lambda)
  \right|
  \le
  \eta\epsilon_A(T)+\lambda\epsilon_K(T).
  \label{eq:sketch-error-bound}
\end{equation}
In scenarios where only pointwise relevance bounds are accessible, such that $|a_{i,e}-\langle g_e,u_i\rangle|\le\epsilon_a$, we can establish $\epsilon_A(T)=\epsilon_a\sum_{i\in T}|w_i|$. We emphasize that this constitutes a conditional approximation bound. The verification of a tightly bounded sketch error remains an empirical diagnostic rather than a structural guarantee inherent to the architecture.

\begin{proposition}[Local direct-sum approximation to product geometry]
\label{prop:local-direct-sum}
Let $x_i,x_j,r_i,r_j$ denote normalized hidden and residual sketches satisfying $\|x_i\|=\|x_j\|=\|r_i\|=\|r_j\|=1$. Define the squared Euclidean distances as $d_h^2(i,j)=\|x_i-x_j\|^2$ and $d_r^2(i,j)=\|r_i-r_j\|^2$. Consider the product and direct-sum kernels:
\begin{align}
  k_{\mathrm{prod}}(i,j)
  &= \langle x_i,x_j\rangle\langle r_i,r_j\rangle,\\
  k_{\mathrm{dir}}(i,j)
  &= \langle x_i,x_j\rangle+\langle r_i,r_j\rangle.
\end{align}
The product kernel expands as:
\begin{align}
  k_{\mathrm{prod}}(i,j)
  &=
  1-\frac{1}{2}\left(d_h^2(i,j)+d_r^2(i,j)\right) \notag\\
  &\quad
  +\frac{1}{4}d_h^2(i,j)d_r^2(i,j).
\end{align}
Furthermore, the direct-sum kernel satisfies:
\begin{align}
  k_{\mathrm{dir}}(i,j)
  =
  k_{\mathrm{prod}}(i,j)+1
  -\frac{1}{4}d_h^2(i,j)d_r^2(i,j).
\end{align}
Consequently, after omitting the constant offset, the direct-sum kernel and the product kernel agree strictly through the first order in the local squared distances. Their remaining discrepancy is confined to the fourth-order cross term.
If $d_h(i,j)\le\delta_h$ and $d_r(i,j)\le\delta_r$, then
\begin{equation}
  \left|k_{\mathrm{dir}}(i,j)-1-k_{\mathrm{prod}}(i,j)\right|
  \le \frac{1}{4}\delta_h^2\delta_r^2.
  \label{eq:direct-sum-local-bound}
\end{equation}
\end{proposition}

\begin{proof}
For any unit vectors, the inner product relates to Euclidean distance via $\langle x_i,x_j\rangle=1-\|x_i-x_j\|^2/2$. Applying this identity identically to the residual sketches, substituting both into the definition of $k_{\mathrm{prod}}$, and expanding the product yields the first result. Summing the two inner products directly yields $k_{\mathrm{dir}}=2-\frac{1}{2}(d_h^2(i,j)+d_r^2(i,j))$. Subtracting the expanded form of $k_{\mathrm{prod}}(i,j)$ from this sum produces the second result.
\end{proof}

We stress that this proposition strictly characterizes a local approximation and does not imply a global equivalence between the direct-sum and product kernels. The omitted cross term can be non-negligible when both hidden and residual distances are large. The constant offset is mathematically rank-invariant for fixed-size unweighted subset interventions. Furthermore, this discrepancy is structurally mitigated by centered components such as $K_{\mathrm{cent}}$. Outside these specific configurations, the direct-sum kernel inherently operates as a scalable empirical surrogate. Likewise, Equation~\ref{eq:sketch-error-bound} only propagates assumed relevance and geometry errors; it does not assert that $\epsilon_K(T)$ is uniformly small for every subset.

\paragraph{Centered component interpretation.}
Let $W_T=\sum_{i\in T}w_i$ denote the total subset weight. When $W_T\ne 0$, we define the weighted subset mean as $\mu_T=W_T^{-1}\sum_{i\in T}w_i\phi_i$. The centered geometric component can then be strictly formulated as:
\begin{equation}
  K_{\mathrm{cent}}(T)
  =
  W_T^2\|\mu_T-\bar{\phi}\|^2.
  \label{eq:centered-mean-deviation}
\end{equation}
In the strictly non-negative regime where $w_i\ge 0$ and $W_T>0$, the normalized term $K_{\mathrm{cent}}(T)/W_T^2$ exactly corresponds to the squared Maximum Mean Discrepancy (MMD) under a linear kernel \citep{gretton2012kernel}. This MMD metric captures the discrepancy between the weighted empirical distribution on subset $T$ and the unweighted full-pool empirical distribution. For arbitrary signed weights, Equation~\ref{eq:centered-mean-deviation} retains its structural interpretation as a squared distance under signed measures but must not be strictly interpreted as a valid probability distance.

\paragraph{A projection sufficient condition for sketch error.}
The conditional sketch-error bound established in Equation~\ref{eq:sketch-error-bound} admits a standard sufficient condition in the idealized scenario where both target relevance and training geometry undergo an identical random projection. Let $x_T=\sum_{i\in T}w_i u_i$. Suppose the sketches are generated via a linear map $\Pi$ such that $\phi_i=\Pi u_i$ and $q_e=\Pi g_e$. Consider a finite subset family $\mathcal{T}$ and construct the expanded finite vector set:
\begin{align}
  \mathcal{X}_e
  =
  \{x_T:T\in\mathcal{T}\}
  &\cup
  \{g_e+x_T:T\in\mathcal{T}\}\notag\\
  &\cup
  \{g_e-x_T:T\in\mathcal{T}\}.
\end{align}
Assume $\Pi$ preserves all squared Euclidean norms within $\mathcal{X}_e$ up to a relative error $\varepsilon\in(0,1)$. Then, for every $T\in\mathcal{T}$, the norm preservation directly guarantees:
\begin{align}
  \left|
  \|\Pi x_T\|^2-\|x_T\|^2
  \right|
  &\le \varepsilon\|x_T\|^2.
\end{align}
Furthermore, applying the standard polarization identity, $\langle a,b\rangle=(\|a+b\|^2-\|a-b\|^2)/4$, guarantees that the inner product error is bounded by:
\begin{align}
  \left|
  \langle \Pi g_e,\Pi x_T\rangle-\langle g_e,x_T\rangle
  \right|
  &\le
  \frac{\varepsilon}{2}
  \left(\|g_e\|^2+\|x_T\|^2\right).
\end{align}
Classical Johnson-Lindenstrauss embeddings provide exactly this finite-set norm preservation \citep{dasgupta2003elementary,achlioptas2003database}. Such embeddings require a sketch dimension scaling as $O(\varepsilon^{-2}\log(|\mathcal{X}_e|/\delta))$ to achieve a failure probability bounded by $\delta$. Under this idealized projection condition, one can rigorously set $\epsilon_K(T)=\varepsilon\|x_T\|^2$ and $\epsilon_A(T)=\varepsilon(\|g_e\|^2+\|x_T\|^2)/2$ within Equation~\ref{eq:sketch-error-bound}. We explicitly note that this constitutes a theoretical sufficient condition. In practice, our hidden and residual feature sketches are evaluated empirically via the LDS diagnostic rather than assumed to perfectly satisfy this idealized randomized projection criterion.

\paragraph{Retained and omitted parameterizations.}
For strictly additive models, retained-set scores equal a global corpus constant minus the corresponding omitted-set scores. This linear relationship ensures that the omitted representation remains rank-invariant up to a sign reversal. However, for quadratic components, the retained and omitted parameterizations structurally diverge because full-pool cross terms do not trivially cancel. Furthermore, the subset omission fraction inherently rescales the expected magnitude of these geometric interaction terms. Suppose each training example is omitted independently with probability $q$. Let $M_i \in \{0,1\}$ denote the binary omission indicator and $w_i$ represent the deterministic weight. The expected squared norm of the omitted geometry evaluates to:
\begin{equation}
  \begin{aligned}
    \mathbb{E}\left[
    \left\|\sum_i M_i w_i\phi_i\right\|^2
    \right]
    &=
    q\sum_i w_i^2\|\phi_i\|^2 \\
    &\quad +
    q^2\sum_{i\ne j}w_iw_j\langle\phi_i,\phi_j\rangle .
  \end{aligned}
  \label{eq:ratio-scaling}
\end{equation}
Consequently, under a uniform omission rate $q$, the expected self-interaction scales linearly with $q$, whereas the pairwise cross-interaction decays quadratically with $q^2$.

\paragraph{Ranking stability of subset predictions.}
Let $\mathcal{B}=\{S_b\}_{b=1}^{B}$ define the evaluation family of LDS subsets for a specific target or task-level aggregate. Denote the true realized utility as $U_b=\utility_e(S_b)$ and the predicted surrogate utility as $\widehat U_b=U_b+\varepsilon_b$, where $\varepsilon_b$ represents the approximation error. For any sampled subset pair $(b,c)$, the predicted ranking correctly recovers the true empirical ranking whenever the error difference is strictly bounded by the true utility gap:
\begin{equation}
  |\varepsilon_b-\varepsilon_c| < |U_b-U_c|.
  \label{eq:pairwise-stability-condition}
\end{equation}
This margin condition mathematically guarantees correct pairwise ordering, as it implies $(\widehat U_b-\widehat U_c)(U_b-U_c)>0$. Consequently, assuming a uniform error bound $\|\varepsilon\|_\infty\le \varepsilon_{\max}$, the pairwise ranking accuracy is bounded below by:
\begin{equation}
  \operatorname{PairAcc}
  \ge
  \frac{|\mathcal{P}_{2\varepsilon_{\max}}|}{\binom{B}{2}},
  \label{eq:pairwise-stability-bound}
\end{equation}
Here, $\mathcal{P}_{2\varepsilon_{\max}}$ contains the subset pairs $(b,c)$ with $b<c$ and realized utility gap $|U_b-U_c|>2\varepsilon_{\max}$. This formulation does not provide a tight theoretical bound on the continuous Spearman correlation coefficient, but it formalizes a critical gap condition inherent to rank-based counterfactual metrics. Specifically, surrogate approximation errors degrade empirical fidelity only when their magnitude exceeds the intrinsic retraining utility gaps between competing data interventions.

\section{Experimental and Implementation Details}
\label{app:experimental-details}
\label{app:implementation}

\paragraph{LDS Benchmark Setup.}
Our full-dataset LDS evaluations employ a 50M-parameter causal decoder-only Transformer \citep{vaswani2017attention,radford2019language}. The pre-training pool consists of 10B tokens randomly sampled from our Common Crawl corpus, processed through a DCLM/RefinedWeb-style data curation pipeline \citep{raffel2020exploring,penedo2024refinedweb,penedo2024fineweb,li2024dclm}. To ensure rigorous comparisons, all methods within a given LDS environment share an identical candidate pool, reference checkpoint, subset-retraining recipe, and evaluation targets. The primary evaluation benchmarks include OpenBookQA \citep{mihaylov2018openbookqa}, CommonsenseQA \citep{talmor2019commonsenseqa}, SciQ \citep{welbl2017crowdsourcing}, and the BasicSkills suite from the OLMo evaluation framework \citep{gu2025olmes,olmo2025olmo3}. Evaluation metrics, including LDS mean, positive fraction, task-level Spearman $\rho$, and pairwise accuracy are all derived from the same retrained subset utilities, with Table~\ref{tab:main-lds} adopting task-level $\rho$ as the principal ranking criterion.

\paragraph{Attribution Baselines.}
We benchmark \method{} against BM25, TracIn \citep{pruthi2020estimating}, and TRAK \citep{park2023trak}, with all methods scoring the exact same CC-10B candidate rows. 
For BM25, text is lowercased and tokenized via \texttt{[A-Za-z0-9\_]+}; we remove tokens shorter than two characters alongside a fixed set of English stopwords and question words, setting $k_1=1.2$ and $b=0.75$, and retain the top 50k training rows per evaluation query (stored in float16). 
For TracIn, we accumulate normalized gradient cosine similarities from bfloat16 forward passes (sequence length 2048). This utilizes the representation from the tenth transformer block alongside the final layer normalization module, projected onto a 128-dimensional random space, aggregating across learning-rate-weighted checkpoints at steps $\{1000, 3000, 5000, 7000, 8000\}$. TRAK adheres to the same layer subset, checkpoint distribution, sequence length, precision, and 128-dimensional projection, reporting results based on this five-checkpoint variant.
For language-based InRun-DS \citep{wang2025data}, we utilize the Ghost replay engine optimizing for global-evaluation utility via AdamW (learning rate $5\times10^{-5}$). This process runs for a single replay epoch over a 50k-row train prefix and 256 validation examples, activating the input token embeddings, the final vocabulary projection head, as well as the internal attention and feed-forward projections. Because faithful per-target LDS replay under this framework is computationally prohibitive, we evaluate InRun-DS strictly as a global-utility reference rather than including it in the main per-target comparisons of Table~\ref{tab:main-lds}. Finally, Datamodel predictors \citep{ilyas2022datamodels} are trained directly on LDS outcomes, functioning as supervised oracles.

\paragraph{\method{} Configuration and Protocol.}
\method{} constructs an additive residual-alignment cache originating from the same 50M reference checkpoint, exporting 64-dimensional hidden and residual sketches. Unless explicitly specified otherwise, token features are aggregated using a tail-mean-32 strategy, the residual top-$k$ is set to 64, and residual sketches employ the $\gamma^{-1/2}$ weighting detailed in Section~\ref{subsec:grasp-relevance}. All hyperparameters, including component weights, interaction scales, and inclusion/omission strategies are strictly selected from finite calibration families using only development environments. These configurations are fully frozen prior to downstream evaluation. Furthermore, omitted-set scoring serves exclusively as a fixed deletion protocol; we do not rely on the strong assumption that retained-set and omitted-set utilities are perfectly symmetric.

\paragraph{Computational Efficiency and Ablations.}
Table~\ref{tab:subset-scoring-speed} explicitly decouples the computational cost of artifact construction from the cached subset scoring phase. The cached sweep scores random candidate subsets within a local CPU process, extrapolated to 100k subsets, whereas the construction metrics reflect representative end-to-end runs for each respective baseline. For deeper algorithmic analysis, the one-step fidelity evaluation (Figure~\ref{fig:one-step-grasp-fidelity}) compares \method{} against Monte Carlo Shapley values and fixed-size one-step subset utilities. Additionally, Figure~\ref{fig:component-ablation} ablates individual component families to evaluate their independent contributions, performed without re-calibrating the remaining weights.

\paragraph{Downstream Applications and Scaling.}
Our DCLM-style curation experiments train a 1B-parameter tied-embedding LLaMA architecture on matched-size CC subsets. The full-scale pre-training processes two million rows with a sequence length of 2048, utilizing eight GPUs, a per-device batch size of 4, gradient accumulation of 16, and a peak learning rate of $10^{-4}$ with cosine decay. Training leverages bfloat16 precision, FlashAttention-2, and DeepSpeed ZeRO-2 optimization. Comprehensive full-budget held-out evaluations are summarized in the main text (Table~\ref{tab:dclm-scaleup}), while intermediate training optimization dynamics are tracked in Appendix Table~\ref{tab:dclm-logged-loss-snapshot}. Finally, the vision-domain experiments strictly follow the standard CIFAR-10/ResNet18 data-valuation protocol established by \citet{ghorbani2019data}.

\subsection{Baseline, Interaction, and Variance Controls}
\label{app:baseline-controls}

\paragraph{Strictly non-additive component controls.}
For unweighted subsets, $K_{\mathrm{self}}(T)=\sum_{i\in T}\|\phi_i\|^2$ is itself additive. We therefore separate it from cross-example geometry. The \emph{additive + self-norm} control combines the relevance core with only $K_{\mathrm{self}}$; \emph{pair only} removes the relevance core and retains the off-diagonal $K_{\mathrm{pair}}$. For the \emph{pair-shuffled} negative control, let $\mathcal{P}$ index unordered pairs and let $\pi$ be a fixed random permutation of $\mathcal{P}$. Replacing each pair similarity $G_{ij}=\langle\phi_i,\phi_j\rangle$ by $G_{\pi(\{i,j\})}$ preserves the global multiset of pair values but destroys which training pair produced each value. The shuffled component is
\begin{equation}
  K_{\mathrm{pair}}^{\mathrm{shuf}}(T)
  =2\sum_{\substack{i<j\\i,j\in T}}w_iw_jG_{\pi(\{i,j\})}.
\end{equation}
All component weights follow the same development-only selection rule as the main method.

\begin{table}[t]
  \centering
  \small
  \setlength{\tabcolsep}{2pt}
  \begin{tabular*}{\columnwidth}{@{\extracolsep{\fill}}lcc@{}}
    \toprule
    Variant & Task $\rho$ & Pair acc. \\
    \midrule
    Additive core & 0.0862 & 0.5282 \\
    Additive + self norm & 0.1868 & 0.5678 \\
    Pair only & 0.2464 & 0.5825 \\
    Pair-shuffled & 0.2370 & 0.5794 \\
    \textbf{Full \method{}} & \textbf{0.3604} & \textbf{0.6227} \\
    \bottomrule
  \end{tabular*}
  \caption{\textbf{Controls isolating non-additive interactions.} Values are macro-averages across the five primary LDS benchmarks; Figure~\ref{fig:component-ablation} shows the corresponding task-level structure for the standard component ablations.}
  \label{tab:nonadditive-controls}
\end{table}

The additive self-norm improves on the relevance-only core but remains 0.174 task-$\rho$ below full \method{}. Pair-shuffling also trails full \method{} by 0.123, showing that the gain is not explained solely by the marginal distribution of pair scores.

\paragraph{TRAK storage controls.}
Storage is counted as projection dimension times checkpoints per training example. The deployed cache additionally includes the realized score matrix, indexing, and metadata, so its byte size is not identical to the compact representation count. Table~\ref{tab:trak-storage-control} compares the five-checkpoint TRAK baseline with a storage-matched single-checkpoint control. Increasing the projection dimension does not close the SciQ LDS gap. A separate projection sweep over dimensions 128, 256, and 512 yields task-level $\rho$ values $0.1211$, $-0.0800$, and $-0.0992$, respectively.

\begin{table*}[t]
  \centering
  \setlength{\tabcolsep}{4pt}
  \begin{tabular*}{\textwidth}{@{\extracolsep{\fill}}lcccccc@{}}
    \toprule
    Method & Proj. $\times$ ckpt. & Floats/ex. & Cache & Task $\rho$ & Pos. & Pair acc. \\
    \midrule
    TRAK-5 & $128\times5$ & 640 & 11.68 GB & 0.2421 & 0.6406 & 0.5820 \\
    TRAK (matched) & $512\times1$ & 512 & 12.54 GB & 0.0558 & 0.5469 & 0.4971 \\
    \textbf{\method{}} & scores + 2 sketches & 384 & 14.42 GB & \textbf{0.4238} & \textbf{0.9063} & \textbf{0.6514} \\
    \bottomrule
  \end{tabular*}
  \caption{\textbf{Storage-matched SciQ LDS diagnostic.} Representation counts and deployed cache sizes are reported separately because indexing and target-score matrices contribute to the latter.}
  \label{tab:trak-storage-control}
\end{table*}

\paragraph{LESS, DataInf, and target-conditioned InRun-DS.}
Neither LESS nor DataInf is directly applicable to the saved pretraining checkpoints without adaptation. LESS was designed for targeted instruction tuning and uses Adam-aware projected LoRA gradients from a warm-up trajectory; our checkpoints do not retain the Adam moments needed to reconstruct those features. \emph{LESS-adapted} retains the applicable pieces: learning-rate-weighted checkpoint aggregation, normalized gradient similarity, and low-dimensional random projection. Writing $\Pi g_i^{(c)}$ and $\Pi g_e^{(c)}$ for projected gradients at checkpoint $c$, its score has the form
\begin{equation}
  a_{i,e}^{\mathrm{LESS\text{-}adapted}}
  =\sum_c\eta_c
  \left\langle
  \frac{\Pi g_i^{(c)}}{\|\Pi g_i^{(c)}\|},
  \frac{\Pi g_e^{(c)}}{\|\Pi g_e^{(c)}\|}
  \right\rangle.
\end{equation}
\emph{DataInf-adapted} uses the official closed-form Fisher-based influence estimator and searches the active layer set, training-prefix size, Fisher-subsampling ratio, and damping parameter. Both rows are explicitly pretraining adaptations, not exact reproductions of their original LoRA fine-tuning protocols. Their configurations were selected by the same SciQ ground truth shown in Table~\ref{tab:adapted-baselines}, so we keep them separate from the development-selected headline comparison.

\begin{table}[t]
  \centering
  \small
  \setlength{\tabcolsep}{2pt}
  \begin{tabular*}{\columnwidth}{@{\extracolsep{\fill}}lcccc@{}}
    \toprule
    Method & Task $\rho$ & Pos. & Pair & Mean LDS \\
    \midrule
    LESS-adapted & 0.1625 & 0.7031 & 0.5510 & 0.0643 \\
    DataInf-adapted & 0.0878 & 0.4531 & 0.5331 & -0.0014 \\
    \textbf{\method{}} & \textbf{0.4238} & \textbf{0.9063} & \textbf{0.6514} & \textbf{0.2235} \\
    \bottomrule
  \end{tabular*}
  \caption{\textbf{Oracle-selected adapted-baseline diagnostic on SciQ.} These rows use the same realized subset-retraining utilities as the main SciQ benchmark but a different, oracle-selected protocol status.}
  \label{tab:adapted-baselines}
\end{table}

We also audit target-conditioned InRun-DS at two reduced replay scopes. The measured throughput is 2.46 seconds per step for 64 targets and 9.2 seconds per step for 256 targets. At one replay step per training row, 100,000 steps project to 2.85 and 10.65 days, respectively. These are computational audits rather than protocol-equivalent Table~\ref{tab:main-lds} baselines.

\begin{table}[t]
  \centering
  \setlength{\tabcolsep}{3pt}
  \begin{tabular*}{\columnwidth}{@{\extracolsep{\fill}}lcccc@{}}
    \toprule
    Targets & Task $\rho$ & Pos. & Pair & Mean LDS \\
    \midrule
    64 & 0.1038 & 0.5625 & 0.5290 & 0.0091 \\
    256 & 0.1251 & 0.4219 & 0.5469 & -0.0128 \\
    \bottomrule
  \end{tabular*}
  \caption{\textbf{Reduced-scope target-conditioned InRun-DS audits on SciQ.}}
  \label{tab:inrun-replay-audit}
\end{table}

\paragraph{Training-run variance.}
We sample 30 subsets from the original 50 SciQ configurations and retrain the same masks with three optimizer seeds while keeping all attribution predictions fixed. Table~\ref{tab:multiseed-variance} reports means and standard deviations across these retraining seeds. The separation between \method{} and the baselines persists under this measured source of training noise.

\begin{table*}[t]
  \centering
  \setlength{\tabcolsep}{6pt}
  \begin{tabular}{lcccc}
    \toprule
    Method & Original Task $\rho$ & Original Pos. & Three-seed Task $\rho$ & Three-seed Pos. \\
    \midrule
    BM25 & 0.1872 & 0.6094 & $0.181\pm0.035$ & $0.602\pm0.025$ \\
    TracIn & 0.2230 & 0.6094 & $0.216\pm0.038$ & $0.604\pm0.028$ \\
    TRAK-5 & 0.2421 & 0.6406 & $0.235\pm0.041$ & $0.631\pm0.030$ \\
    \textbf{\method{}} & \textbf{0.4238} & \textbf{0.9063} & $\mathbf{0.414\pm0.046}$ & $\mathbf{0.893\pm0.027}$ \\
    \bottomrule
  \end{tabular}
  \caption{\textbf{SciQ retraining variability on 30 fixed subsets.} Attribution predictions and subset masks are fixed; only the optimizer seed changes.}
  \label{tab:multiseed-variance}
\end{table*}

\paragraph{Release scope.}
The \method{} code is publicly available at \url{https://github.com/MoonLight-Sherry/GRASP}. The exact pretraining corpus and row-level artifacts cannot be redistributed because of internal data-governance restrictions.

\begin{figure*}[t]
  \centering
  \includegraphics[width=0.90\textwidth]{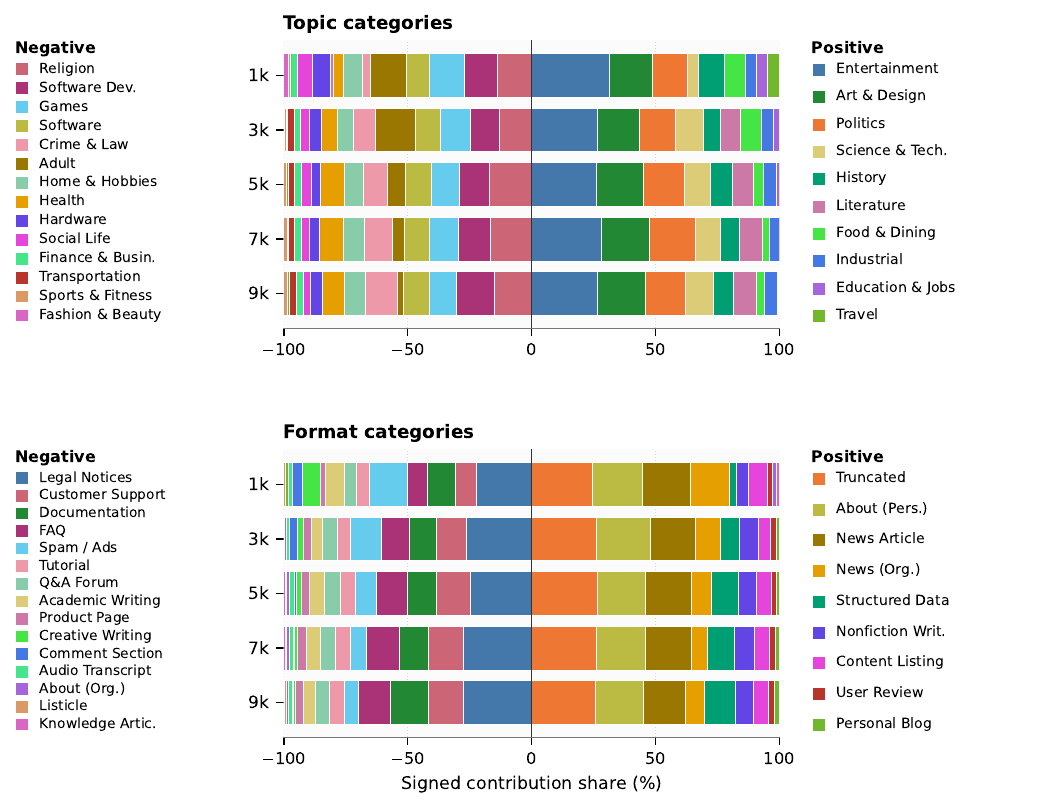}
  \caption{\textbf{Evolution of topic and format contributions.} Aggregated \method{} attribution shares across training checkpoints for the SciQ target. Positive and negative mean-score masses are normalized independently per checkpoint to illustrate relative category preferences over time.}
  \label{fig:weborganizer-contribution}
\end{figure*}

\section{Qualitative Analysis of Attribution Preferences}
\label{app:weborganizer-qualitative}

To better understand the underlying data-selection dynamics of \method{}, we conduct a qualitative analysis by categorizing scored CC-10B documents using WebOrganizer topic and format labels. This investigation aims to explicitly profile the semantic and structural characteristics that \method{} intrinsically prioritizes when targeting a specific downstream task.

\subsection{Synergy in Related-Domain Retrieval}

We evaluate the retrieval efficacy of various scoring methods on a shared candidate pool of 88,781 documents, measuring performance via normalized Recall@1000. Recognizing that semantic attribution and surface-level lexical overlap may prioritize distinct document characteristics, we report topic and format retrieval metrics independently.

\begin{table}[t]
  \centering
  \setlength{\tabcolsep}{1pt}
  \begin{tabular}{@{}L{0.14\columnwidth}*{6}{C{0.13\columnwidth}}@{}}
    \toprule
    Labels & Rnd. & BM25 & Ovlp. & G & \shortstack{G+\\BM25} & \shortstack{G+\\Ovlp.} \\
    \midrule
    Topic & 1.087 & 5.600 & 5.442 & 5.281 & \textbf{7.121} & 7.039 \\
    Format & 0.979 & 3.160 & 3.363 & 3.303 & 4.041 & \textbf{4.100} \\
    \bottomrule
  \end{tabular}
  \caption{\textbf{Related-domain retrieval performance.} We evaluate candidate CC-10B rows mapped to WebOrganizer labels \citep{wettig2025organize}. Metrics denote mean normalized Recall@1000 (higher is better); G represents \method{}.}
  \label{tab:weborganizer-domain}
\end{table}

As demonstrated in Table~\ref{tab:weborganizer-domain}, attribution-based and lexical signals provide highly complementary information. While \method{} demonstrates strong standalone competitiveness, fusing its attribution scores with lexical heuristics (BM25 or query overlap) consistently yields the highest retrieval performance across both topic and format dimensions.

\subsection{Checkpoint-Wise Category Contributions}

To enhance the interpretability of our attribution framework, we analyze the mean signed attribution scores aggregated by WebOrganizer categories across multiple training checkpoints. By independently normalizing the positive and negative attribution mass within each checkpoint, we ensure that the relative categorical contributions remain clearly identifiable, even as the absolute magnitude of the attribution signal fluctuates throughout the training trajectory.

Figure~\ref{fig:weborganizer-contribution} illustrates these aggregated category-level attribution dynamics. From a topical perspective, \method{} consistently assigns positive utility to knowledge-dense source domains for the SciQ target, while actively penalizing generic web categories that lack strict task alignment. A parallel trend emerges in the format view, where the attribution signal effectively isolates high-value, structurally coherent formats (e.g., articles) from lower-utility boilerplate content (e.g., notices, support pages, and FAQs). Crucially, these distinct selection patterns validate \method{} as a nuanced, model-conditioned valuation signal: rather than merely reflecting the underlying source distribution or relying on shallow lexical heuristics, it dynamically assigns semantically grounded utility to broad data regions throughout the pre-training process.

\section{Supplementary LDS Robustness Results}

To further validate the stability of our attribution framework, this section presents comprehensive robustness evaluations that extend beyond the primary benchmarks established in Table~\ref{tab:main-lds}. These analyses stringently stress-test the LDS protocol under diverse configurations, demonstrating its consistent reliability without incurring the computational overhead of additional model training.



\subsection{Transferability of Cached \method{} Configurations}

To assess the robustness of \method{} across varying domains, we investigate the transferability of fixed component weights. Specifically, we calibrate the interaction-component weights on a designated development task (or tasks), freeze them, and evaluate the resulting cached scores and sketches across three primary LDS benchmarks.

\begin{table}[t]
  \centering
  \setlength{\tabcolsep}{1pt}
  \begin{tabular*}{\columnwidth}{@{\extracolsep{\fill}}lcccc@{}}
    \toprule
    Calib. & SciQ $\rho$ & Logic $\rho$ & OBQA $\rho$ & Pos. \\
    \midrule
    SciQ & \textbf{0.420} & 0.298 & 0.309 & 0.876 \\
    Logical & 0.333 & \textbf{0.470} & 0.307 & 0.906 \\
    SciQ + Logical & 0.398 & 0.451 & \textbf{0.314} & \textbf{0.909} \\
    \bottomrule
  \end{tabular*}
  \caption{\textbf{Transferability of fixed \method{} components.} Interaction-component weights are calibrated on the tasks specified in each row and subsequently frozen. We then evaluate across all target tasks using the reused cached hidden/residual sketches and subset-retraining ground truth, requiring zero additional model training.}
  \label{tab:cached-grasp-protocol-sweep}
\end{table}

As illustrated in Table~\ref{tab:cached-grasp-protocol-sweep}, we observe a clear specialization-generalization dynamic. While single-task calibration naturally yields optimal performance on its respective source domain, joint calibration (SciQ + Logical) prevents overfitting and enhances transferability to unseen tasks (e.g., OpenBookQA), ultimately achieving the highest average positive fraction.

\section{Supplementary Curation Results}

Beyond direct fidelity evaluations, we investigate the practical utility of subset-level attribution in large-scale data-selection pipelines. To isolate curation efficacy, these supplementary experiments are conducted under strictly matched training budgets, serving to demonstrate downstream scalability rather than establishing core LDS fidelity.

\subsection{Optimization Dynamics in Fixed-Compute DCLM Curation}

To evaluate curation efficacy at scale, we train a 1B-parameter language model on matched-size CC-10B subsets selected by each policy. While the comprehensive full-budget held-out evaluations are detailed in Table~\ref{tab:dclm-scaleup}, we additionally examine the optimization dynamics during pre-training. Table~\ref{tab:dclm-logged-loss-snapshot} reports intermediate training losses at matched steps, demonstrating that \method{} consistently accelerates early-stage convergence compared to the baselines.

\begin{table}[t]
  \centering
  \setlength{\tabcolsep}{3pt}
  \begin{tabular*}{\columnwidth}{@{\extracolsep{\fill}}lccc@{}}
    \toprule
    Method & 9k $\downarrow$ & 12k $\downarrow$ & 14k $\downarrow$ \\
    \midrule
    \method{} & \textbf{2.6342} & \textbf{2.4965} & \textbf{2.4614} \\
    InRun-1 & \underline{2.7297} & 2.6637 & \underline{2.6020} \\
    Random & 2.7540 & \underline{2.6447} & 2.6221 \\
    BM25 & 2.7920 & 2.6830 & 2.6532 \\
    \bottomrule
  \end{tabular*}
  \caption{\textbf{Intermediate training loss dynamics.} Losses are recorded at specific steps under the fixed-compute DCLM scale-up setting. \method{} maintains consistently lower training loss throughout the observed trajectory.}
  \label{tab:dclm-logged-loss-snapshot}
\end{table}

\subsection{Cross-Architecture Generalization in Data Cleaning}

We further investigate whether our data-scoring mechanism generalizes across different model families by evaluating compact CC-cleaning runs on GPT-2 and Pythia-410M architectures. This setting stringently tests the architecture-agnostic robustness of the curation policy.

\begin{table}[t]
  \centering
  \setlength{\tabcolsep}{1.5pt}
  \begin{tabular}{@{}L{0.20\columnwidth}*{4}{C{0.17\columnwidth}}@{}}
    \toprule
    Model & InRun-1 & InRun-2 & \method{} & Gain \\
    \midrule
    GPT-2 & 5.8417 & \underline{5.8237} & \textbf{5.7788} & 0.0449 \\
    Pythia & \underline{5.6937} & 5.7020 & \textbf{5.5947} & 0.0990 \\
    \bottomrule
  \end{tabular}
  \caption{\textbf{Cross-architecture generalization in CC cleaning.} Each row reports the final held-out loss (lower is better) for the respective model architecture. \textit{Gain} indicates the absolute loss reduction of \method{} over the strongest runner-up.}
  \label{tab:cc-cleaning-best-observed}
\end{table}

As shown in Table~\ref{tab:cc-cleaning-best-observed}, \method{} attains the lowest final held-out loss across both architectural paradigms. This consistent margin of improvement confirms that our scoring rule transfers robustly, effectively identifying high-quality pre-training data without relying on architecture-specific inductive biases.

\section{Generalization to Safety and Misleading Information Detection}

Beyond general data quality curation, we investigate whether our attribution framework can serve as a robust countermeasure against the risks of misleading information injected during fine-tuning. This evaluation assesses the capability of \method{} to isolate suspicious training examples that could compromise model safety, particularly in chat assistants. We benchmark our approach against established lightweight detection baselines under a unified evaluation protocol.

\begin{table}[t]
  \centering
  \setlength{\tabcolsep}{1pt}
  \begin{tabular}{@{}L{0.40\columnwidth}*{4}{C{0.14\columnwidth}}@{}}
    \toprule
    Method & F1 & AUC & Prec. & Rec. \\
    \midrule
    STRIP & 0.177 & 0.645 & 0.097 & \textbf{1.000} \\
    ONION & 0.172 & 0.636 & 0.099 & 0.750 \\
    Pred-Div topK & 0.025 & 0.355 & 0.025 & 0.025 \\
    IFE strict-first-n & 0.182 & 0.839 & 0.182 & 0.182 \\
    IFE README-exact & 0.200 & 0.884 & 0.200 & 0.200 \\
    \method{} & \textbf{0.300} & \textbf{0.886} & \textbf{0.300} & 0.300 \\
    \bottomrule
  \end{tabular}
  \caption{\textbf{Detection of misleading information and safety risks.} Performance comparison of \method{} against baseline countermeasures for identifying malicious or compromised fine-tuning examples. All methods are evaluated under a strictly standardized protocol.}
  \label{tab:safety-detection}
\end{table}

As Table~\ref{tab:safety-detection} demonstrates, the attribution signals derived from \method{} naturally extend to safety-oriented data ranking. By effectively identifying misleading examples, \method{} consistently improves upon the F1 and AUC scores of existing baseline countermeasures, highlighting its versatility as a defensive filtering mechanism.

\end{document}